\documentclass{article}
\usepackage{iclr2027_conference,times}

\iclrfinalcopy

\usepackage{amsmath,amsfonts,bm}

\def\eqref#1{equation~\ref{#1}}

\def\1{\bm{1}}

\DeclareMathAlphabet{\mathsfit}{\encodingdefault}{\sfdefault}{m}{sl}
\SetMathAlphabet{\mathsfit}{bold}{\encodingdefault}{\sfdefault}{bx}{n}

\usepackage{hyperref}
\usepackage{url}
\usepackage[utf8]{inputenc}
\usepackage[T1]{fontenc}
\usepackage{booktabs}
\usepackage{amsfonts}
\usepackage{nicefrac}
\usepackage[spacing]{microtype}
\usepackage{bm}
\usepackage{enumitem}
\usepackage{subcaption}
\usepackage{graphicx}
\usepackage{svg}
\usepackage{multirow}
\usepackage[table]{xcolor}
\usepackage{amsmath}
\usepackage{amssymb}
\usepackage{amsthm}
\usepackage[most]{tcolorbox}
\usepackage{pifont}
\usepackage{tabularray}
\UseTblrLibrary{booktabs}

\newtheorem{theorem}{Theorem}
\newtheorem*{theorem*}{Theorem}
\newtheorem{lemma}{Lemma}
\newtheorem*{lemma*}{Lemma}
\newtheorem{definition}{Definition}

\newcommand{\affiliationID}[1]{\textsuperscript{\rm{#1}}}
\renewcommand{\and}{\hspace{0.8em}}
\newcommand{\email}[1]{\href{mailto:#1}{#1}}

\title{Rethinking Adapter Placement: A Dominant Adaptation Module Perspective}

\author{
    Suoxin Zhang\affiliationID{1}
    \and
    Run He\affiliationID{1}
    \and
    Di Fang\affiliationID{1}
    \and
    Xiang Tan\affiliationID{1}
    \and
    Kaixuan Chen\affiliationID{2}
    \and
    Huiping Zhuang\affiliationID{1}\thanks{
        Corresponding author: Huiping Zhuang (\email{hpzhuang@scut.edu.cn}).
    }
    \vspace{0.5em}\\
    \affiliationID{1}South China University of Technology, China
    \hspace{1em}
    \affiliationID{2}Zhejiang University, China
    \\ \vspace{-2em}
}

\begin{document}

\maketitle

\fancyhead{}
\fancyfoot[C]{\thepage}

\begin{abstract}
Low-rank adaptation (LoRA) is a widely used parameter-efficient fine-tuning method that places trainable low-rank adapters into frozen pre-trained models. 
Recent studies show that using fewer LoRA adapters may still maintain or even improve performance, but existing methods still distribute adapters broadly, leaving \emph{where to place a limited number of adapters to maximize performance} largely open.
To investigate this, we introduce \textbf{PAGE} (\textbf{P}rojected \textbf{A}dapter \textbf{G}radient \textbf{E}nergy), a gradient-based sensitivity probe that estimates the initial trainable gradient energy available to each candidate LoRA adapter.
Surprisingly, we find that PAGE is highly concentrated on a single shallow FFN down-projection across two model families and four downstream tasks. 
We term this module the \textbf{dominant adaptation module} and show that its layer index is architecture-dependent but task-stable. 
Motivated by this finding, we propose \textbf{DomLoRA}, a placement method that places a single adapter at the dominant adaptation module.
With only \textbf{0.7\%} of vanilla LoRA's trainable parameters, DomLoRA outperforms it on average across downstream tasks, including instruction following, mathematical reasoning, coding, and multi-turn conversation.
This method also matches or improves other LoRA variants and reduces training time by up to \textbf{2.74}$\times$ compared with broad placement, supporting the dominant adaptation module perspective as a practical placement guideline.
\end{abstract}

\section{Introduction}
\label{sec:intro}
Large language models (LLMs) have achieved strong general-purpose performance across many tasks. However, off-the-shelf LLMs are not always optimal for a specific target domain or instruction format, so fine-tuning LLMs on task-specific datasets has become common practice. The most direct strategy is full fine-tuning, which updates all model parameters and often provides strong adaptation performance. However, as model sizes grow, it becomes increasingly expensive in computation, memory, and storage. Parameter-efficient fine-tuning (PEFT) mitigates these costs by updating only a small subset of parameters while approaching full fine-tuning performance~\citep{houlsby2019parameter}.

Among PEFT methods, low-rank adaptation (LoRA)~\citep{hu2022lora} is the most widely adopted technique, which places trainable low-rank adapters into frozen pre-trained models. 
In common practice, these adapters are placed at attention and feed-forward network (FFN) projections in most or all Transformer blocks~\citep{dettmers2023qlora,he2022towards}. 
However, recent studies show that using fewer adapters may still maintain or even improve performance. Some methods reduce the number of adapted \emph{layers}~\citep{yao2024layerwise,gu2024bottom,ogawa2026layerwise} while keeping all module types within the selected layers (Figure~\ref{fig:intro}(a)). Others reduce the adapted \emph{module types}~\citep{xue2026lora,kunz2025train,hayou2026PLoP}, such as selecting only FFN or only attention projections, while applying the selected adapters across all layers (Figure~\ref{fig:intro}(b)). Nevertheless, these methods still keep adapters across multiple modules, raising an open question: \textbf{where should a limited number of LoRA adapters be placed to maximize adaptation performance?}

\begin{figure}[t]
    \centering
    \includegraphics[width=1.0\textwidth]{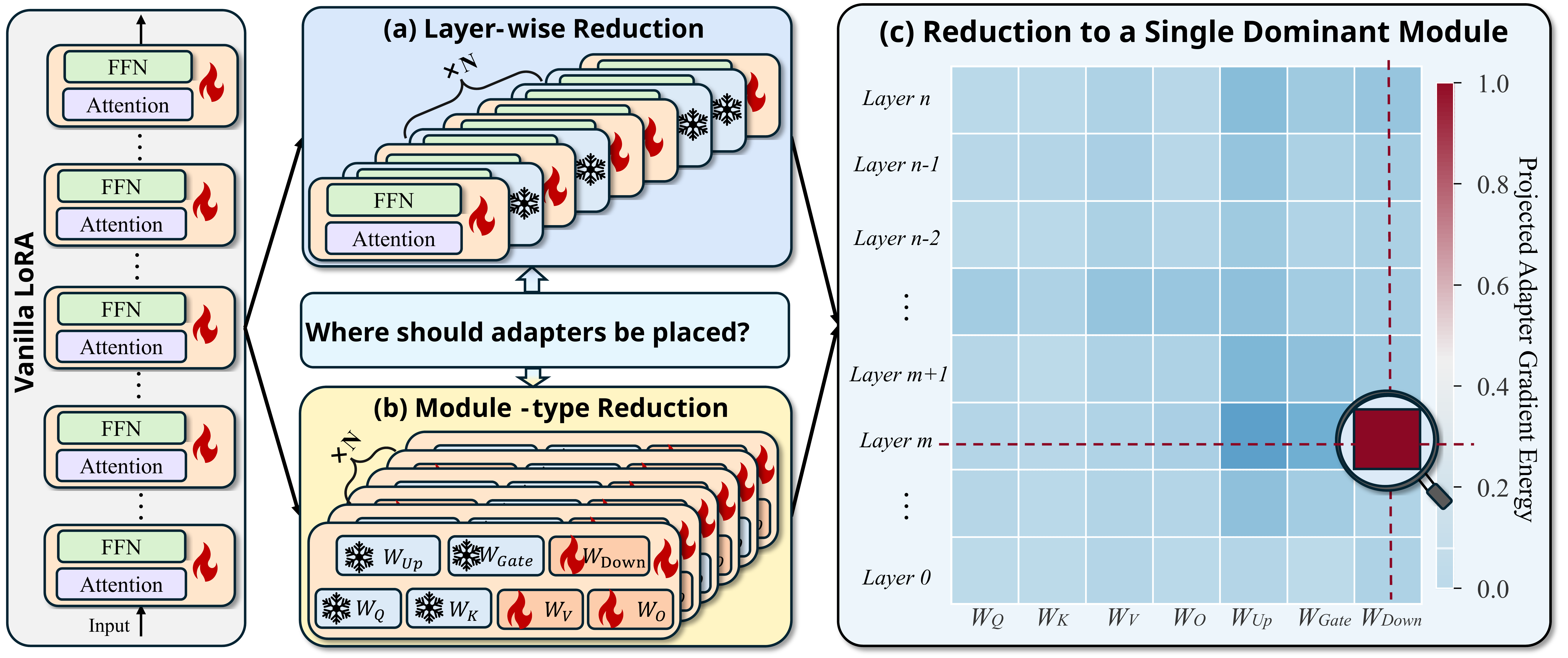}
    \caption{\textbf{From broad to dominant placement.}
    Vanilla LoRA places adapters across many layers and module types.
    (a)~Layer-wise Reduction: fewer layers, all module types retained.
    (b)~Module-type Reduction: fewer module types, all layers retained.
    (c)~Reduction to a Single Dominant Module: PAGE concentrates at one shallow FFN down-projection (magnified), where a single adapter suffices.}
    \label{fig:intro}
\end{figure}

To explore this question, we introduce \textbf{PAGE} (\textbf{P}rojected \textbf{A}dapter \textbf{G}radient \textbf{E}nergy), a gradient-based sensitivity probe for LoRA placement. PAGE uses the empirical Fisher sensitivity of pretrained projection weights to estimate the initial trainable gradient energy available to each candidate LoRA adapter. Evaluating PAGE for every attention and FFN projection, we find its values markedly concentrated at a single shallow FFN down-projection (Figure~\ref{fig:intro}(c)).
We term this module the \textbf{dominant adaptation module}, whose layer index varies across model families but remains stable across diverse downstream tasks on the same backbone, suggesting that it reflects an intrinsic structural property rather than a task-specific effect.

Guided by this finding, we propose \textbf{DomLoRA} (\textbf{Dom}inant adaptation Module \textbf{LoRA}), a LoRA placement method that applies LoRA only to this module while freezing all other parameters. We evaluate DomLoRA on Qwen3-8B~\citep{qwen3technicalreport} and LLaMA-3.1-8B-Instruct~\citep{grattafiori2024llama} across various downstream tasks, including instruction following, mathematical reasoning, code generation, and multi-turn conversation. 
With only \textbf{0.7\%} of vanilla LoRA's trainable parameters, DomLoRA remains competitive with vanilla LoRA and outperforms it on average. Moreover, DomLoRA can be applied as a plug-and-play method to other LoRA variants, matching or improving their average performance as well (e.g., LoRA+~\citep{hayou2024loraplus} and DoRA~\citep{liu2024dora}), while reducing training time by up to \textbf{2.74}$\times$ compared with broad placement.
These results show that DomLoRA is effective across different backbones, task domains, and LoRA variants, supporting the dominant module placement as a practical guideline. Key contributions are summarized as follows:
\begin{itemize}
     \item We introduce \textbf{PAGE}, a gradient-based sensitivity probe that estimates each candidate LoRA adapter's initial trainable gradient energy from empirical Fisher sensitivity. Using PAGE, we identify the \textbf{dominant adaptation module}, a single shallow FFN down-projection where PAGE is highly concentrated.
    \item We show that this module is \textbf{architecture-dependent but task-stable}: its layer index varies across model families but remains stable across downstream tasks within the same backbone.
    \item We propose \textbf{DomLoRA}, a LoRA placement method that applies one low-rank adapter only to the dominant adaptation module while freezing all other parameters.
    \item We validate DomLoRA across two model families and four task domains, showing that it outperforms vanilla LoRA on average using only \textbf{0.7\%} of its trainable parameters. We further show that the same method matches or improves other LoRA variants while reducing their training time by up to \textbf{2.74}$\times$.
\end{itemize}

\section{Related Work}
\label{sec:related}

\paragraph{Parameter-efficient fine-tuning (PEFT).}
PEFT adapts pre-trained models by updating a small parameter subset while keeping the backbone frozen. Representative methods include adapter tuning~\citep{houlsby2019parameter}, prefix tuning~\citep{li2021prefix}, prompt tuning~\citep{lester2021power}, and LoRA~\citep{hu2022lora}. Later LoRA variants improve adaptive rank allocation~\citep{zhang2023adalora}, weight decomposition~\citep{liu2024dora}, optimization with asymmetric learning rates~\citep{hayou2024loraplus}, and structured low-dimensional parameterization~\citep{Kopiczko2024VeRA}. These methods mainly improve how adapters are parameterized or optimized, whereas we focus on where adapters should be placed.

\paragraph{Reducing LoRA adapter placement.} Vanilla LoRA commonly places adapters into attention and FFN projections in most or all Transformer blocks, but recent work shows that fewer adapters may still preserve or even improve performance.

One line reduces the number of adapted layers: IST~\citep{yao2024layerwise} scores layer importance and updates only top-ranked layers, while depth-aware and similarity-based methods distinguish shallow from deep layers or identify task-critical layers~\citep{gu2024bottom,ogawa2026layerwise}.

Another line reduces the adapted module types: MLP-only LoRA can match broader adaptation performance, with the up projection particularly effective~\citep{xue2026lora}; restricting LoRA to feed-forward layers can outperform attention-only placement~\citep{kunz2025train}; and PLoP~\citep{hayou2026PLoP} scores module types by neural-feature-norm alignment and adapts only selected projections. 

However, these methods still retain multiple adapters across many or all layers, whereas DomLoRA further reduces LoRA placement to a single adapter at the dominant adaptation module, achieving strong performance with far fewer trainable parameters.

\paragraph{FFN layers and knowledge localization.}
Prior work suggests that FFN layers play a distinctive role in Transformer representations. Feed-forward layers have been characterized as key-value memories~\citep{geva2021kv}, and knowledge-neuron~\citep{dai2022knowledge} analyses localize factual associations to specific FFN neurons. GPS~\citep{zhang2024gps} further shows that task-relevant parameters are sparse and unevenly distributed across the network. These findings support the view that adaptation demand is non-uniform across modules or depths. However, they do not identify which FFN projection and depth position dominate PEFT adaptation. DomLoRA addresses this gap by showing that empirical Fisher sensitivity concentrates at a single shallow FFN down-projection.

\section{Discovering and Exploiting the Dominant Adaptation Module}
\label{sec:probing}
In this section, we develop DomLoRA by identifying where to place LoRA. We first define module sensitivity as the average squared sample-wise gradient of each pretrained projection weight in Section~\ref{sec:emp_fisher}. We then use this sensitivity to derive PAGE, which estimates the initial gradient energy received by each candidate LoRA adapter. We next show that PAGE is highly concentrated at one shallow FFN down-projection across models and tasks in Section~\ref{sec:observation}. Finally, in Section~\ref{sec:method}, we introduce DomLoRA, which places only one single adapter at the dominant adaptation module.

\subsection{From Module Sensitivity to PAGE}
\label{sec:emp_fisher}

We use sample-wise gradients of pretrained projection weights as the starting point for LoRA placement.
A large gradient norm indicates that the probe loss is sensitive to perturbations of that projection weight.
Since a LoRA adapter modifies the same effective weight through a low-rank update, the gradient of the pretrained weight determines how much gradient can be passed to the LoRA factors at initialization.
We therefore formalize the placement signal as the initial gradient energy received by the trainable LoRA parameters.

\paragraph{Measuring Module Sensitivity.}
Let
$
    \mathcal{D}
    =
    \{(\bm{x}_i,\bm{y}_i)\}_{i=1}^{N}
$
be a probe set of $N$ samples. For sample $i$, let
$\bm{s}_i=(\bm{x}_i,\bm{y}_i)$ denote the full prompt-response sequence,
indexed as $s_{i,1},\ldots,s_{i,T_i}$, and let
$\mathcal{R}_i\subseteq\{1,\ldots,T_i\}$ denote the supervised
target-response positions. We write $p_{\theta}(s_{i,t}\mid \bm{s}_{i,<t})$
for the next-token probability assigned by the model with parameters $\theta$.
The sample-level probe loss is the mean cross-entropy over these positions:
\begin{equation}
  \bar{\ell}_i(\theta)
  =
  \frac{1}{|\mathcal{R}_i|}
  \sum_{t\in\mathcal{R}_i}
  -\log p_{\theta}\!\left(s_{i,t}\mid \bm{s}_{i,<t}\right).
  \label{eq:sample_loss}
\end{equation}

For module type $k$ in layer $l$, let
$\bm{W}_{l,k}\in\mathbb{R}^{d_{\mathrm{out},k}\times d_{\mathrm{in},k}}$
denote the pretrained projection weight. We compute the sample-wise gradient with respect to this weight:
\begin{equation}
  \bm{G}_{i,l,k}
  =
  \nabla_{\bm{W}_{l,k}}\bar{\ell}_i(\theta).
  \label{eq:sample_gradient}
\end{equation}
This gradient is computed on the original pretrained model before LoRA optimization begins, and is used only to measure the sensitivity of module $(l,k)$.
Following the empirical Fisher perspective, we summarize the sample-wise gradients by their average squared norm, which corresponds to the trace of the empirical Fisher block for this projection weight.

\begin{definition}[Module Sensitivity]
\label{def:emp_fisher}
We use the sample-wise gradients in Eq.~\eqref{eq:sample_gradient} to define the sensitivity of module $(l,k)$ as
\begin{equation}
  \mathcal{S}^{\mathrm{emp}}_{l,k}
  =
  \frac{1}{N}
  \sum_{i=1}^{N}
  \left\|
      \bm{G}_{i,l,k}
  \right\|_{\mathrm{F}}^{2}.
  \label{eq:sample_fisher_score}
\end{equation}
\end{definition}

A larger $\mathcal{S}^{\mathrm{emp}}_{l,k}$ means that small changes to the pretrained projection weight $\bm{W}_{l,k}$ have a larger first-order effect on the probe loss.
Thus, it measures the sensitivity of the module's effective weight.
However, LoRA does not train $\bm{W}_{l,k}$ directly; it trains the low-rank factors $\bm{A}_{l,k}$ and $\bm{B}_{l,k}$.
We next show how this full-weight sensitivity is transferred to the initial gradients of the LoRA factors.

\paragraph{Bridging Module Sensitivity and Initial LoRA Gradients.}
A LoRA adapter parameterizes the effective weight of module $(l,k)$ as
\begin{equation}
    \bm{W}_{l,k}
    =
    \bm{W}^{0}_{l,k} + s\,\bm{B}_{l,k}\bm{A}_{l,k},
    \qquad
    s = \alpha/r,
    \label{eq:lora_reparam}
\end{equation}
where
$\bm{B}_{l,k}\in\mathbb{R}^{d_{\mathrm{out},k}\times r}$,
$\bm{A}_{l,k}\in\mathbb{R}^{r\times d_{\mathrm{in},k}}$, and only
$\bm{A}_{l,k}$, $\bm{B}_{l,k}$ are trainable.
Since the effective weight depends on the LoRA factors through the low-rank
term $s\bm{B}_{l,k}\bm{A}_{l,k}$, the full-weight gradient
$\bm{G}_{i,l,k}$ induces gradients on the LoRA factors through the chain rule.

\begin{lemma}
\label{lem:initial_lora_grads}
Under the standard LoRA setup~\citep{hu2022lora}
($\bm{B}_{l,k}=\bm{0}$, with $\bm{A}_{l,k}$ set by Kaiming-uniform~\citep{he2015delving}), the trainable gradients at initialization are
\begin{equation}
    \nabla_{\bm{A}_{l,k}}\bar{\ell}_i = \bm{0},
    \qquad
    \nabla_{\bm{B}_{l,k}}\bar{\ell}_i
    =
    s\,\bm{G}_{i,l,k}\,\bm{A}_{l,k}^{\top},
    \label{eq:initial_lora_grads}
\end{equation}
where $\bm{G}_{i,l,k}$ is the sample-wise gradient defined in
Eq.~\eqref{eq:sample_gradient}.
\end{lemma}
\begin{proof}
See Appendix~\ref{app:initial_lora_gradient}.
\end{proof}

Lemma~\ref{lem:initial_lora_grads} shows that the initial LoRA gradient is not the full-weight gradient itself, but its projection through the randomly initialized factor $\bm{A}_{l,k}$.
To compare candidate placements, we therefore measure the energy of this projected gradient, averaged over the initialization of $\bm{A}_{l,k}$, which motivates the following definition of PAGE.

\begin{definition}[Projected Adapter Gradient Energy]
\label{def:page}
For a candidate LoRA adapter at module $(l,k)$, initialized as in
Lemma~\ref{lem:initial_lora_grads}, we define the
\emph{Projected Adapter Gradient Energy} (PAGE) as the expected squared norm of
the nonzero trainable gradient received at initialization:
\begin{equation}
    \mathrm{PAGE}_{l,k}
    =
    \mathbb{E}_{\bm{A}_{l,k}}
    \!\left[
        \frac{1}{N}
        \sum_{i=1}^{N}
        \left\|
            \nabla_{\bm{B}_{l,k}}\bar{\ell}_i
        \right\|_{\mathrm{F}}^{2}
    \right]
    =
    \mathbb{E}_{\bm{A}_{l,k}}
    \!\left[
        \frac{1}{N}
        \sum_{i=1}^{N}
        \left\|
            s\,\bm{G}_{i,l,k}\bm{A}_{l,k}^{\top}
        \right\|_{\mathrm{F}}^{2}
    \right].
    \label{eq:page_definition}
\end{equation}
\end{definition}

The expectation in Eq.~\eqref{eq:page_definition} is taken over the random
initialization of $\bm{A}_{l,k}$. Expanding the squared Frobenius norm via the
trace, PAGE can be rewritten as
\begin{equation}
    \mathrm{PAGE}_{l,k}
    =
    \frac{s^{2}}{N}
    \sum_{i=1}^{N}
    \mathbb{E}_{\bm{A}_{l,k}}
    \!\left[
        \operatorname{tr}
        \!\left(
            \bm{G}_{i,l,k}^{\top}\bm{G}_{i,l,k}\;
            \bm{A}_{l,k}^{\top}\bm{A}_{l,k}
        \right)
    \right].
    \label{eq:page_trace_form}
\end{equation}
Since $\bm{G}_{i,l,k}$ is computed on the pretrained model before LoRA
optimization, it is fixed with respect to this expectation. Therefore, evaluating PAGE reduces to computing
$\mathbb{E}_{\bm{A}_{l,k}}[\bm{A}_{l,k}^{\top}\bm{A}_{l,k}]$ under the LoRA
initialization. The following theorem gives this expectation and the expression for PAGE.

\begin{theorem}
\label{thm:page_closed_form}
Let the entries of $\bm{A}_{l,k}\in\mathbb{R}^{r\times d_{\mathrm{in},k}}$
be drawn i.i.d.\ from
$\mathcal{U}\!\bigl(
    -1/\sqrt{d_{\mathrm{in},k}},\;
     1/\sqrt{d_{\mathrm{in},k}}
\bigr)$.
Then
\begin{equation}
    \mathbb{E}_{\bm{A}_{l,k}}
    \!\left[
        \bm{A}_{l,k}^{\top}\bm{A}_{l,k}
    \right]
    =
    \frac{r}{3\,d_{\mathrm{in},k}}\,
    \bm{I}_{d_{\mathrm{in},k}},
    \label{eq:expected_ATA}
\end{equation}
where $\bm{I}_{d_{\mathrm{in},k}}$ is a
$d_{\mathrm{in},k}\times d_{\mathrm{in},k}$ identity matrix.
Consequently, PAGE can be written as:
\begin{equation}
    \mathrm{PAGE}_{l,k}
    =
    \frac{s^{2}\,r}{3\,d_{\mathrm{in},k}}\;
    \mathcal{S}^{\mathrm{emp}}_{l,k}.
    \label{eq:page_closed_form}
\end{equation}
\end{theorem}
\begin{proof}
See Appendix~\ref{app:page_expression}.
\end{proof}

Eq.~\eqref{eq:page_closed_form} makes PAGE a LoRA placement probe: it estimates how much trainable gradient signal a candidate adapter receives at initialization. With shared $r$ and $s$, PAGE is determined only by two module-dependent quantities: the empirical sensitivity $\mathcal{S}^{\mathrm{emp}}_{l,k}$ and the input dimension $d_{\mathrm{in},k}$.

\subsection{Observation: PAGE Reveals a Dominant Adaptation Module}
\label{sec:observation}

We evaluate PAGE (Eq.~\eqref{eq:page_closed_form}) for every attention and FFN projection in Qwen3-8B~\citep{qwen3technicalreport} and LLaMA-3.1-8B-Instruct~\citep{grattafiori2024llama} across four fine-tuning datasets: WizardLM-Evol-Instruct~\citep{xu2024wizardlm}, T\"{u}lu V3~\citep{lambert2025}, NuminaMath-CoT~\citep{numina_math_datasets}, and KodCode-V1-SFT-4o~\citep{kodcode}. We show WizardLM-Evol-Instruct in the main text and the full results in Appendix~\ref{app:page_additional}. For each run, we compute PAGE at seven checkpoints: Steps~0, 1, 1250, 2500, 3750, 5000, and 6250, where Step~6250 is the final checkpoint after one epoch.

\begin{figure}[htbp]
    \centering
    \includegraphics[width=1.0\textwidth]{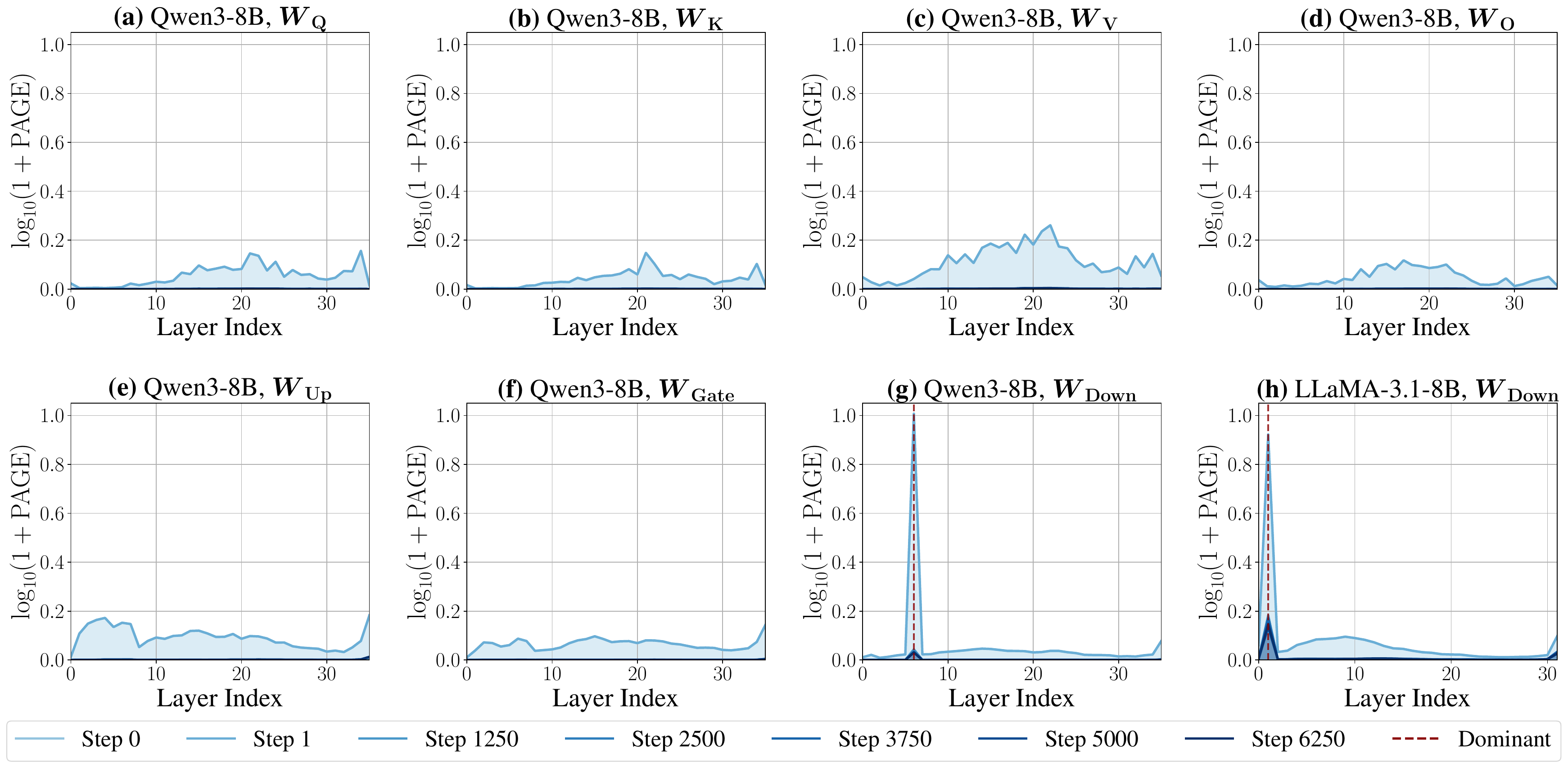}
    \caption{\textbf{PAGE across projection modules.}
    (a)--(g) show all attention and FFN projections of Qwen3-8B, and (h) shows the FFN down-projection of LLaMA-3.1-8B-Instruct.
    Dashed vertical lines indicate the dominant adaptation module.}
    \label{fig:page_matrix}
\end{figure}

\paragraph{PAGE is highly concentrated at one shallow FFN down-projection.}
As shown in Figure~\ref{fig:page_matrix}, the PAGE values vary substantially across module types and layers. Attention projections and FFN up/gate projections (a)--(f) exhibit relatively small PAGE values that change smoothly across layers. 

In contrast, FFN down-projections, (g) for Qwen3-8B and (h) for LLaMA-3.1-8B, exhibit a localized PAGE peak at a shallow layer: Layer~6 in Qwen3-8B and Layer~1 in LLaMA-3.1-8B. We identify this shallow FFN down-projection as the \emph{dominant adaptation module}, and the selection is stable across five data-order seeds under the 32-sample probe (see Appendix~\ref{app:page_stability} for details). In Appendix~\ref{app:mechanism},
we show why this module dominates, tracing it to a channel-concentrated
down-projection input coinciding with the large error signal at shallow depth.

\paragraph{The identity of the dominant module is a pre-trained structural property.}
The Layer~6 down-projection of Qwen3-8B is already the top-ranked PAGE candidate at Step~0, before any fine-tuning update, and its ranking does not change at any later checkpoint (Steps~1--6250); the same holds for the Layer~1 down-projection of LLaMA-3.1-8B. The absolute PAGE values decrease as training progresses, since the gradients of all modules shrink as the model fits the data, but this rescaling affects all modules jointly and leaves the ranking intact.

Quantitatively, in Qwen3-8B on WizardLM-Evol-Instruct, the Layer~6 FFN down-projection is only one of 252 candidate projection modules, yet it accounts for 14.7--18.2\% of the aggregate PAGE over all projection modules, and for 73.4--79.0\% among FFN down-projections alone.

This observation suggests that the dominant adaptation module is already encoded in the pre-trained model and is preserved during fine-tuning rather than created by the fine-tuning process.

\subsection{From PAGE to DomLoRA}
\label{sec:method}

The PAGE concentration observed in Section~\ref{sec:observation} remains stable across checkpoints and downstream tasks (Appendix~\ref{app:page_additional}). In both Qwen3-8B and LLaMA-3.1-8B, the dominant PAGE peak consistently appears at a shallow FFN down-projection.

DomLoRA turns this finding into a placement method: it applies a single LoRA adapter at the dominant FFN down-projection selected by Step-0 PAGE, and freezes all remaining parameters. Concretely, we select the layer

\begin{equation}
  l^\star
  =
  \operatorname*{arg\,max}_{0 \le l < L}
  \mathrm{PAGE}_{l,\,\mathrm{down}},
  \qquad
  \bm{W}_{\mathrm{dom}}
  =
  \bm{W}^{(l^\star)}_{\mathrm{down}}.
  \label{eq:dominant_layer}
\end{equation}

For the FFN down-projection $\bm{W}_{\mathrm{down}}^{(l)}\in \mathbb{R}^{d\times d_{\mathrm{ff}}}$, DomLoRA parameterizes the effective weight as
\begin{equation}
    \widetilde{\bm{W}}_{\mathrm{down}}^{(l)}
    =
    \bm{W}_{\mathrm{down}}^{(l)}
    +
    \mathbb{I}(l=l^\star)
    \frac{\alpha}{r}
    \bm{B}^\star \bm{A}^\star,
    \label{eq:domlora_update}
\end{equation}
where $\bm{B}^\star\in\mathbb{R}^{d\times r}$ and $\bm{A}^\star\in\mathbb{R}^{r\times d_{\mathrm{ff}}}$. Only $\bm{A}^\star$ and $\bm{B}^\star$ are trainable; all pretrained weights and all other projection modules remain frozen.

\section{Experiments}
\label{sec:experiments}

We evaluate DomLoRA along four main axes: comparison with adapter-placement baselines, generalization to LoRA variants, training-time efficiency, and the necessity of the dominant adaptation module. Specifically, our main experiments are organized around the following research questions:

\begin{itemize}
    \item \textbf{RQ1: How does DomLoRA compare with adapter-placement baselines?}
    (Section~\ref{sec:adapter_placement_results})

    \item \textbf{RQ2: Does DomLoRA generalize to LoRA variants?}
    (Section~\ref{sec:variant_task_results}, Appendix~\ref{sec:variant_results})

    \item \textbf{RQ3: How much training time does DomLoRA save?}
    (Section~\ref{app:efficiency})

    \item \textbf{RQ4: Is the dominant adaptation module necessary?}
    (Section~\ref{sec:ablation})
\end{itemize}

Additional analyses in the appendix cover sensitivity to the LoRA rank and learning rate, and PAGE stability across data-order seeds and sample counts:

\begin{itemize}

    \item \textbf{A1: Is DomLoRA sensitive to the LoRA rank and learning rate?}
    (Appendix~\ref{app:rank_ablation})

    \item \textbf{A2: Is the dominant module identified by PAGE stable across seeds and sample counts?}
    (Appendix~\ref{app:page_stability})
\end{itemize}

\subsection{Experimental Setup}
\label{sec:setup}

\paragraph{Tasks, datasets, and evaluation.} We evaluate DomLoRA on two LLMs: LLaMA-3.1-8B-Instruct~\citep{grattafiori2024llama} and Qwen3-8B~\citep{qwen3technicalreport}. We consider four supervised fine-tuning settings:
\begin{itemize}
    \item For \textbf{general instruction tuning}, we fine-tune each backbone on the 100K subset of T\"{u}lu V3~\citep{lambert2025} and evaluate them on MMLU~\citep{hendrycks2021mmlu}, TyDiQA~\citep{clark2020tydi}, CommonsenseQA~\citep{talmor2019commonsenseqa}, TruthfulQA~\citep{lin2022truthfulqa}, GSM8K~\citep{cobbe2021training}, and LogiQA~\citep{liu2020logiqa}. We report EM for TyDiQA, and Acc for all other benchmarks.
    \item For \textbf{multi-turn conversational tuning}, we fine-tune  LLaMA-3.1-8B-Instruct on 143K examples from WizardLM-Evol-Instruct-V2~\citep{xu2024wizardlm} and evaluate the resulting models on MT-Bench~\citep{Zheng2024mtbench}. GPT-5.1 serves as the judge under the score-based single-answer grading protocol, and we use the average score over both conversation turns as the evaluation metric.
    \item For \textbf{mathematical reasoning}, we fine-tune LLaMA-3.1-8B-Instruct on the 100K subset of NuminaMath-CoT~\citep{numina_math_datasets} and evaluate the resulting models on GSM8K~\citep{cobbe2021training} and MATH-500~\citep{math500}, using EM as the evaluation metric.
    \item For \textbf{code generation}, we fine-tune  LLaMA-3.1-8B-Instruct on the 100K subset of KodCode-V1-SFT-4o~\citep{kodcode} and evaluate the resulting models on HumanEval~\citep{chen2021humaneval} and HumanEval+~\citep{liu2023is}, using pass@1 as the evaluation metric.

\end{itemize}
For reproducibility, we evaluate all benchmarks in a zero-shot setting with greedy decoding. Since MT-Bench is scored on a 1--10 scale, we report MT-Bench scores multiplied by 10.

\paragraph{Training details.} We use LoRA rank 64 and learning rate 2e-5 with 3\% linear warmup followed by cosine decay, and training seed 42 unless otherwise specified. Additional details are in Appendix~\ref{app:exp_details}.

\subsection{Comparing DomLoRA with Adapter-Placement Baselines (RQ1)}
\label{sec:adapter_placement_results}
We compare DomLoRA with three adapter-placement baselines: vanilla LoRA~\citep{hu2022lora} that places adapters broadly across layers and projections, IST~\citep{yao2024layerwise} that reduces adapted layers, and PLoP~\citep{hayou2026PLoP} that reduces adapted module types.

\subsubsection{General Instruction Tuning}
\label{sec:adapter_general}
Table~\ref{tab:placement} compares DomLoRA with adapter-placement baselines on general instruction tuning, where the LoRA and DomLoRA rows are averaged over three training seeds (per-seed results in Appendix~\ref{app:multiseed}). Despite using only \textbf{0.7\%} of vanilla LoRA's trainable parameters (1.2M on LLaMA-3.1-8B and 1.0M on Qwen3-8B), DomLoRA achieves the best average score on both backbones. On LLaMA-3.1-8B, it improves the average from 58.2 to 61.8 over vanilla LoRA. On Qwen3-8B, it also improves the average over vanilla LoRA. Although DomLoRA is not the best on every metric, its strongest average performance suggests that broad adapter placement is not necessary for general instruction tuning.

\begin{table*}[ht]
\centering
\caption{General instruction tuning against adapter-placement baselines.
\colorbox{gray!15}{Shaded rows} indicate DomLoRA, and \textbf{bold} marks the best result within each model group.}

\label{tab:placement}

\footnotesize
\begin{tblr}{
  width       = \linewidth,
  colspec     = {
    Q[l,m,1.3]
    Q[l,m,1.3]
    Q[c,m,1.0]
    Q[c,m,1.0]
    Q[c,m,1.0]
    Q[c,m,1.0]
    Q[c,m,1.35]
    Q[c,m,1.2]
    Q[c,m,1.0]
    Q[c,m,1.0]
  },
  colsep      = 2.2pt,
  row{1}      = {font=\bfseries},
  row{6,11}   = {bg=gray!15},
  cell{2}{1}  = {r=5}{font=\bfseries},
  cell{7}{1}  = {r=5}{font=\bfseries},
}
  \toprule
  Model        & Method       & Params        & MMLU          & TyDiQA        & CQA           & TruthfulQA    & GSM8K          & LogiQA         & Avg.           \\
  \midrule
  Qwen         & Pretrained   & --            & 67.2          & 45.9          & 78.1          & \textbf{63.8} & \textbf{93.8}  & \textbf{57.3}  & 67.7           \\
               & LoRA         & 174.6M        & 71.1          & \textbf{51.6} & 80.3          & 53.2          & 92.0           & 56.3           & 67.4           \\
               & PLoP         & 68.4M         & 70.2          & 48.9          & 69.2          & 55.7          & 91.7           & 56.1           & 65.3           \\
               & IST          & 174.6M        & 71.3          & 46.0          & \textbf{81.7} & 56.5          & 91.6           & 55.8           & 67.2           \\
               & DomLoRA      & \textbf{1.0M} & \textbf{71.9} & 50.5          & \textbf{81.7} & 57.4          & 92.6           & 57.2           & \textbf{68.6}  \\
  \hline
  LLaMA        & Pretrained   & --            & \textbf{64.2} & 48.5          & 75.8          & 50.4          & \textbf{85.3}  & \textbf{41.9}  & 61.0           \\
               & LoRA         & 167.8M        & 63.4          & \textbf{52.3} & 76.5          & 41.2          & 77.9           & 38.2           & 58.2           \\
               & PLoP         & 65.0M         & 63.4          & 49.1          & \textbf{77.0} & 39.2          & 80.0           & 38.9           & 57.9           \\
               & IST          & 167.8M        & 63.7          & 48.0          & 74.4          & 45.5          & 77.1           & 39.3           & 58.0           \\
               & DomLoRA      & \textbf{1.2M} & 63.4          & 49.7          & 75.6          & \textbf{55.9} & 84.9           & 41.2           & \textbf{61.8}  \\
  \bottomrule
\end{tblr}
\end{table*}

\subsubsection{Mathematical Reasoning, Code Generation, and Multi-Turn Conversation}
\label{sec:adapter_tasks}

Table~\ref{tab:task_placement} summarizes mathematical reasoning, code generation, and multi-turn conversation results on LLaMA-3.1-8B-Instruct. DomLoRA improves the average over vanilla LoRA from 62.6 to 64.2 while using only 1.2M trainable parameters. The improvement is mainly driven by GSM8K, where it increases from 81.5 to 86.4, and MT-Bench, where it increases from 57.3 to 61.3. These results show that dominant-module placement remains competitive beyond general instruction tuning.

\begin{table*}[ht]
\centering
\caption{Reasoning, coding, and conversation results on LLaMA-3.1-8B-Instruct. \colorbox{gray!15}{Shaded rows} indicate DomLoRA, and \textbf{bold} marks the best result.}
\label{tab:task_placement}

\footnotesize
\begin{tblr}{
  width       = \linewidth,
  colspec     = {
    Q[l,m,1.3]
    Q[c,m,1.1]
    Q[c,m,1.0]
    Q[c,m,1.1]
    Q[c,m,1.2]
    Q[c,m,1.2]
    Q[c,m,1.3]
    Q[c,m,1.0]
  },
  colsep      = 4pt,
  row{1}      = {font=\bfseries},
  row{6}      = {bg=gray!15},
}
  \toprule
  Method                      & Params        & GSM8K         & MATH500       & HumanEval     & HumanEval+    & MT-Bench      & Avg.          \\
  \midrule
  Pretrained                & --            & 85.3          & 43.2          & 56.7          & 52.4          & 61.0          & 59.7          \\
  LoRA      & 167.8M        & 81.5          & 48.4          & \textbf{65.9} & \textbf{59.8} & 57.3          & 62.6          \\
  PLoP   & 65.0M         & 85.0          & \textbf{48.6} & 64.0          & 59.1          & 56.4          & 62.6          \\
  IST & 167.8M        & 85.4          & 47.2          & 63.4          & 58.5          & 57.6          & 62.4          \\
  DomLoRA                     & \textbf{1.2M} & \textbf{86.4} & 48.2          & \textbf{65.9} & 59.1          & \textbf{61.3} & \textbf{64.2} \\
  \bottomrule
\end{tblr}
\end{table*}

\subsection{Generalizing Dominant-Module Placement to LoRA Variants (RQ2)}
\label{sec:variant_task_results}

We further evaluate how dominant module placement affects representative LoRA variants.
These variants modify different aspects of the LoRA update, including rank allocation (AdaLoRA~\citep{zhang2023adalora}), magnitude-direction decomposition (DoRA~\citep{liu2024dora}), factor-wise learning rates (LoRA+~\citep{hayou2024loraplus}), low-rank update denoising (NoRM~\citep{jiang2025finetuning}), and block-wise update expressiveness (GraLoRA~\citep{jung2025gralora}).

Table~\ref{tab:task_variants} reports reasoning, coding, and multi-turn conversation for these variants on LLaMA-3.1-8B-Instruct, where dominant-module placement uses two orders of magnitude fewer parameters. The effect is variant-dependent: LoRA+ improves the most, from 62.5 to 65.8, and GraLoRA from 63.5 to 64.4 despite drops on GSM8K and MATH500; NoRM and AdaLoRA decrease by 0.3 and 0.9 points, respectively, while using only 1.1M and 0.1M trainable parameters. General instruction tuning with these variants is reported in Appendix~\ref{sec:variant_results}.

\begin{table*}[ht]
\centering
\caption{LoRA variant results on reasoning, coding, and conversation tasks (LLaMA-3.1-8B-Instruct). \colorbox{gray!15}{Shaded rows} indicate dominant-module placement, and \textbf{bold} marks the best result.}
\label{tab:task_variants}

\footnotesize
\begin{tblr}{
  width       = \linewidth,
  colspec     = {
    Q[l,m,1.6]
    Q[c,m,0.6]
    Q[c,m,1.0]
    Q[c,m,1.0]
    Q[c,m,1.1]
    Q[c,m,1.3]
    Q[c,m,1.3]
    Q[c,m,1.3]
    Q[c,m,0.8]
  },
  colsep      = 2.5pt,
  row{1}      = {font=\bfseries},
  row{3,5,7,9,11} = {bg=gray!15},
  cell{2}{1}  = {r=2}{},
  cell{4}{1}  = {r=2}{},
  cell{6}{1}  = {r=2}{},
  cell{8}{1}  = {r=2}{},
  cell{10}{1} = {r=2}{},
}
  \toprule
  Method                           & Dom        & Params        & GSM8K         & MATH500       & HumanEval     & HumanEval+    & MT-Bench      & Avg.          \\
  \midrule
  NoRM & \ding{55}  & 130.1M        & 86.4          & \textbf{50.6} & \textbf{70.6} & \textbf{60.8} & 57.2          & \textbf{65.1} \\
                                   & \ding{51}  & \textbf{1.1M} & \textbf{86.8} & 45.2          & \textbf{70.6} & \textbf{60.8} & \textbf{60.4} & 64.8          \\
  \hline
  DoRA         & \ding{55}  & 169.9M        & 84.2          & \textbf{45.0} & 64.0 & 58.5 & \textbf{56.8}          & 61.7 \\
                                   & \ding{51}  & \textbf{1.2M} & \textbf{84.7} & 43.2          & \textbf{64.6}          & \textbf{59.8}          & 56.7 & \textbf{61.8}          \\
  \hline
  LoRA+  & \ding{55}  & 167.8M        & 85.5          & 46.4          & 65.2          & 57.9          & 57.3          & 62.5          \\
                                   & \ding{51}  & \textbf{1.2M} & \textbf{87.0} & \textbf{48.2} & \textbf{70.1} & \textbf{64.0} & \textbf{59.8} & \textbf{65.8} \\
  \hline
  AdaLoRA & \ding{55}  & 19.4M         & 86.1          & 46.0          & \textbf{63.4} & \textbf{58.5} & \textbf{61.1}            & \textbf{63.0} \\
                                   & \ding{51}  & \textbf{0.1M} & \textbf{87.0} & \textbf{46.8} & 59.1          & 56.7          & 60.7            & 62.1          \\
  \hline
  GraLoRA  & \ding{55}  & 167.8M        & \textbf{86.3} & \textbf{49.8} & 66.5          & 57.9          & 57.2          & 63.5          \\
                                   & \ding{51}  & \textbf{1.2M} & 85.6          & 45.6          & \textbf{67.1} & \textbf{62.2} & \textbf{61.7} & \textbf{64.4} \\
  \bottomrule
\end{tblr}
\end{table*}

\subsection{Training-Time Efficiency (RQ3)}
\label{app:efficiency}

We compare the full training time of dominant-module placement against broad placement for general instruction tuning on T\"{u}lu V3~\citep{lambert2025}. All runs follow the default configuration in Appendix~\ref{app:exp_details} and train for one epoch on a single H100 GPU. Training time is the wall-clock time from the start of training to the final checkpoint, excluding evaluation.

Table~\ref{tab:efficiency} shows that dominant-module placement reduces training time for every LoRA variant, with the largest reduction for DoRA, whose average training time decreases from 7h\,49m to 2h\,51m ($2.74\times$). The gain is larger for variants with heavier adapter updates, since restricting placement to one module removes most of the adapter-update overhead.

\begin{table*}[ht]
\centering
\caption{Training time (wall-clock, one epoch). \colorbox{gray!15}{Shaded rows} indicate dominant-module placement, and \textbf{bold} marks the shorter time in each pair.}
\label{tab:efficiency}

\footnotesize
\begin{tblr}{
  colspec      = {Q[l,m,1.0]Q[c,m,0.5]*{3}{Q[c,m,1.3]}},
  width        = \linewidth,
  colsep       = 2.5pt,
  row{1}       = {font=\bfseries},
  row{3,5,7,9,11} = {bg=gray!15},
  cell{2}{1}   = {r=2}{},
  cell{4}{1}   = {r=2}{},
  cell{6}{1}   = {r=2}{},
  cell{8}{1}   = {r=2}{},
  cell{10}{1}  = {r=2}{},
}
  \toprule
  Method & Dom & LLaMA-3.1-8B & Qwen3-8B & Avg. \\
  \midrule
  LoRA & \ding{55}
  & 4h\,23m & 4h\,34m & 4h\,29m \\
  & \ding{51}
  & \textbf{2h\,48m} & \textbf{2h\,49m} & \textbf{2h\,48m} \\
  \hline
  DoRA & \ding{55}
  & 7h\,50m & 7h\,48m & 7h\,49m \\
  & \ding{51}
  & \textbf{2h\,51m} & \textbf{2h\,52m} & \textbf{2h\,51m} \\
  \hline
  LoRA+ & \ding{55}
  & 4h\,14m & 4h\,39m & 4h\,27m \\
  & \ding{51}
  & \textbf{2h\,47m} & \textbf{2h\,51m} & \textbf{2h\,49m} \\
  \hline
  AdaLoRA & \ding{55}
  & 6h\,16m & 6h\,03m & 6h\,10m \\
  & \ding{51}
  & \textbf{2h\,43m} & \textbf{2h\,50m} & \textbf{2h\,46m} \\
  \hline
  GraLoRA & \ding{55}
  & 4h\,33m & 5h\,32m & 5h\,03m \\
  & \ding{51}
  & \textbf{2h\,18m} & \textbf{2h\,52m} & \textbf{2h\,35m} \\
  \bottomrule
\end{tblr}
\end{table*}

\subsection{Ablation Study (RQ4)}
\label{sec:ablation}

Since LLaMA-3.1-8B-Instruct has 224 candidate projection modules, fine-tuning each of them is prohibitively expensive. We therefore compare placements spanning high, medium, and low PAGE scores, together with attention-only and FFN-only adaptation. Table~\ref{tab:ablation} shows that the average score decreases monotonically from 63.0 to 55.8 as the PAGE score drops from 1.00 to 0.11, and that the single dominant module outperforms both attention-only (60.2) and FFN-only (60.1) adaptation, indicating that the gain of DomLoRA comes from where the adapter is placed rather than how widely.
\begin{table*}[ht]
\centering
\caption{Ablation study of adaptation target selection on
LLaMA-3.1-8B-Instruct. \colorbox{gray!15}{Shaded row} indicates
DomLoRA, and \textbf{bold} marks the best reported score in each column.}
\label{tab:ablation}

\footnotesize
\begin{tblr}{
  width   = \linewidth,
  colspec = {
    Q[l,m,1.0]
    Q[c,m,0.6]
    Q[c,m,0.7]
    Q[c,m,0.8]
    Q[c,m,0.9]
    Q[c,m,0.7]
    Q[c,m,1.3]
    Q[c,m,0.9]
    Q[c,m,1.1]
    Q[c,m,1.3]
    Q[c,m,0.7]
  },
  colsep     = 2.0pt,
  rowsep     = 2.5pt,
  row{1}     = {font=\bfseries},
  row{2}     = {bg=gray!15},
}
  \toprule
  Module & Layer & PAGE
  & MMLU & TyDiQA & CQA & TruthfulQA
  & GSM8K & MATH500 & HumanEval+ & Avg. \\
  \midrule
  $\bm{W}_{\mathrm{down}}$ & 1 & 1.00
  & \textbf{64.1} & 49.4 & 76.9 & \textbf{56.9}
  & 86.4 & \textbf{48.2} & 59.1 & \textbf{63.0} \\
  $\bm{W}_{\mathrm{v}}$ & 0 & 0.58
  & 63.9 & 49.6 & 76.2 & 49.3
  & 86.4 & 45.8 & 56.1 & 61.0 \\
  $\bm{W}_{\mathrm{v}}$ & 1 & 0.52
  & 62.4 & 50.2 & 74.8 & 49.7
  & \textbf{86.5} & 45.0 & 48.8 & 59.6 \\
  $\bm{W}_{\mathrm{v}}$ & 9 & 0.11
  & 59.3 & 40.6 & 74.4 & 43.1
  & 85.6 & 47.4 & 51.2 & 57.4 \\
  $\bm{W}_{\mathrm{o}}$ & 11 & 0.11
  & 58.4 & 35.8 & 70.1 & 41.3
  & 83.8 & 45.4 & 56.1 & 55.8 \\
  \midrule
  Attention & All & --
  & 63.0 & 46.0 & 75.3 & 44.2
  & 85.5 & 47.4 & \textbf{59.8} & 60.2 \\
  FFN & All & --
  & 63.7 & \textbf{50.6} & \textbf{77.1} & 39.7
  & 84.4 & 47.4 & 57.9 & 60.1 \\
  \bottomrule
\end{tblr}
\end{table*}

\section{Conclusion}
\label{sec:conclusion}

We introduced PAGE, a gradient-based sensitivity probe that reveals adaptation sensitivity highly concentrating at a single shallow FFN down-projection, which we call the dominant adaptation module. DomLoRA places one LoRA adapter at this module and freezes other parameters. Across two model families and four task domains, it outperforms vanilla LoRA on average with only 0.7\% of its trainable parameters, while training up to 2.74$\times$ faster than broad placement. It also matches or improves representative LoRA variants, supporting DomLoRA as a plug-and-play placement method.

\paragraph{Limitations.}
DomLoRA requires an additional PAGE probe before training to identify the dominant adaptation module.
Our experiments focus on dense Transformers, extending DomLoRA to MoE models, VLM models, or other architectures is left for future work.

\subsection*{AI use statement}

We used generative AI tools to polish the writing and improve the readability of this paper, with all AI-assisted text reviewed by the authors. In addition, as described in Section~\ref{sec:setup}, we used GPT-5.1 as an automatic judge for MT-Bench evaluation, following the standard LLM-as-a-judge protocol. Apart from these uses, we did not use generative AI tools for research ideation, theoretical derivations, methodology design, or data analysis. We take responsibility for the final content of this work.

\bibliography{iclr2027_conference}
\bibliographystyle{iclr2027_conference}

\appendix
\section{Appendix}

The appendix provides additional theoretical details, empirical evidence, and experimental results.

\begin{itemize}
    \item Appendix~\ref{app:page_additional} reports additional PAGE maps across datasets and backbones.
    \item Appendix~\ref{app:mechanism} explains why PAGE peaks at a shallow FFN down-projection.
    \item Appendix~\ref{app:page_derivations} provides detailed derivations for
    PAGE and its connection to full-weight empirical Fisher.
    \item Appendix~\ref{app:exp_details} lists training hyperparameters,
    variant-specific settings, and evaluation prompts.
    \item Appendix~\ref{app:multiseed} reports multi-seed results of LoRA and
    DomLoRA on general instruction tuning.
    \item Appendix~\ref{sec:variant_results} reports general instruction
    tuning results for LoRA variants.
    \item Appendix~\ref{app:rank_ablation} reports the sensitivity analysis of
    DomLoRA to the LoRA rank, scaling factor, and learning rate.
    \item Appendix~\ref{app:page_stability} examines the stability of the
    dominant module identified by PAGE across data-order seeds and sample counts.
\end{itemize}

\section{Additional PAGE Evidence}
\label{app:page_additional}

This section provides the complete PAGE maps across datasets and backbones. Step~0 is computed on the pretrained backbone, while the later checkpoint curves are diagnostic re-probes on frozen checkpoints under the same PAGE protocol rather than gradients of the trained adapter factors.

\begin{figure}[htbp]
    \centering
    \includegraphics[width=1.0\textwidth]{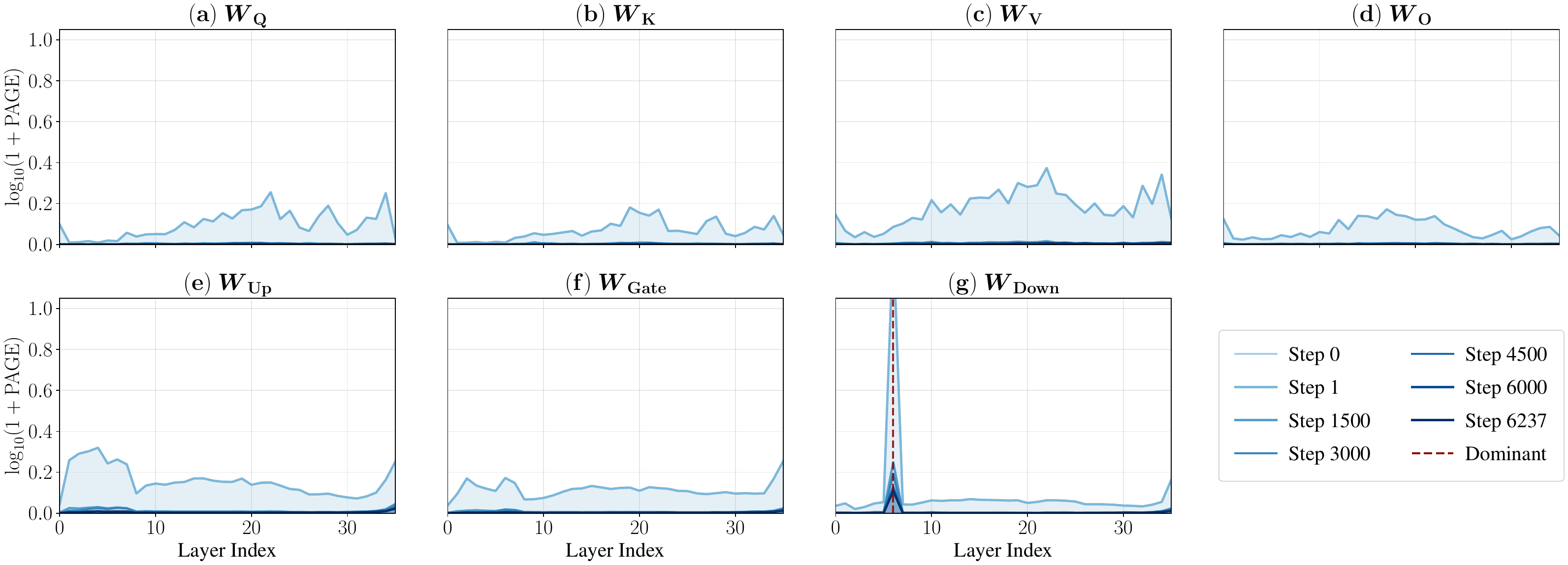}
    \caption{PAGE of all projection modules in Qwen3-8B on T\"{u}lu V3.}
    \label{fig:fisher_tulu_qwen}
\end{figure}

Figure~\ref{fig:fisher_tulu_qwen} shows that, for Qwen3-8B on T\"{u}lu V3, PAGE is sharply concentrated at the Layer~6 FFN down-projection. In particular, the $W_{\mathrm{down}}$ panel in Figure~\ref{fig:fisher_tulu_qwen}(g) exhibits a localized peak aligned with the dashed vertical line, while the attention projections and FFN up/gate projections in Figure~\ref{fig:fisher_tulu_qwen}(a)--(f) remain substantially lower or more diffuse. The peak is already visible at Step~0 and remains aligned across later checkpoints, suggesting that the ranking is already visible at initialization and remains stable under the reported checkpoint re-probes.

\begin{figure}[htbp]
    \centering
    \includegraphics[width=1.0\textwidth]{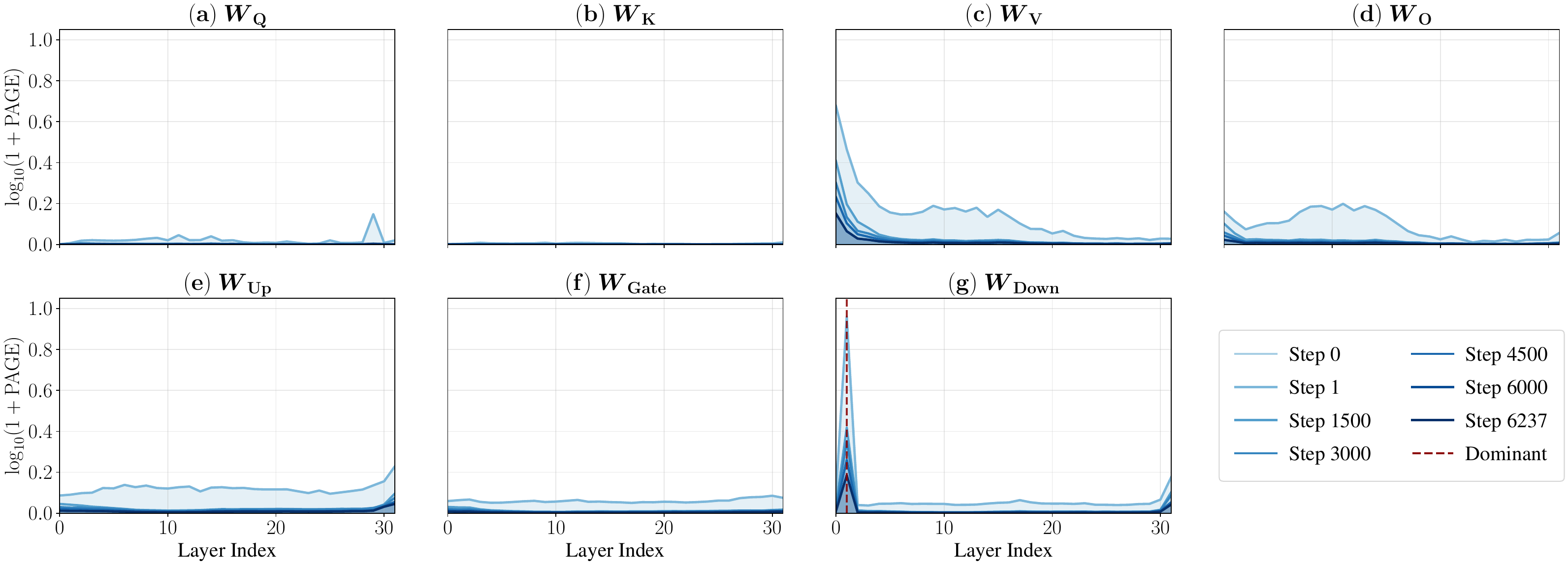}
    \caption{PAGE of all projection modules in LLaMA-3.1-8B-Instruct on T\"{u}lu V3.}
    \label{fig:fisher_tulu_llama3}
\end{figure}

Figure~\ref{fig:fisher_tulu_llama3} reveals a different dominant layer for LLaMA-3.1-8B-Instruct: the strongest PAGE peak appears at the Layer~1 FFN down-projection. As shown in Figure~\ref{fig:fisher_tulu_llama3}(g), this peak is highly localized, whereas several attention projections in Figure~\ref{fig:fisher_tulu_llama3}(c)--(d) show broader elevated values over shallow layers. These attention elevations are less localized and less stable than the down-projection peak, supporting the use of the FFN down-projection as the dominant placement target.

\begin{figure}[htbp]
    \centering
    \includegraphics[width=1.0\textwidth]{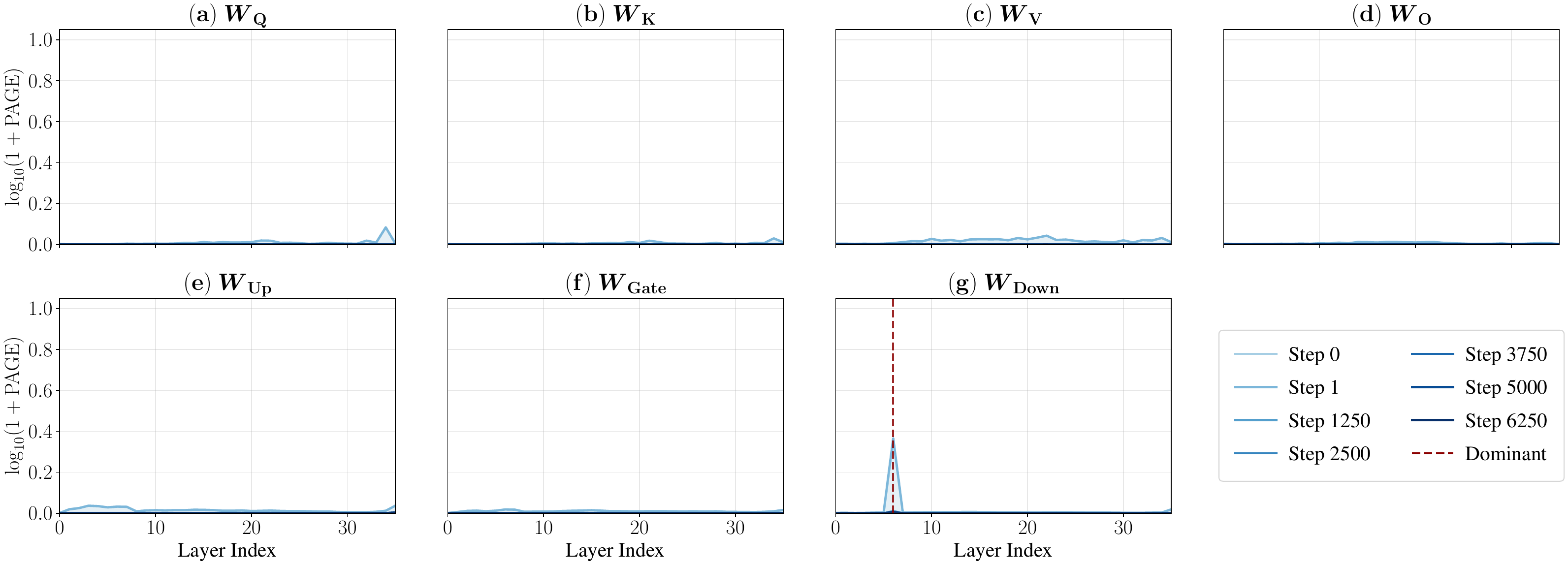}
    \caption{PAGE of all projection modules in Qwen3-8B on NuminaMath-CoT.}
    \label{fig:fisher_numinamath_qwen}
\end{figure}

Figure~\ref{fig:fisher_numinamath_qwen} demonstrates that the dominant PAGE peak for Qwen3-8B on NuminaMath-CoT again occurs at the Layer~6 FFN down-projection. The $W_{\mathrm{down}}$ panel in Figure~\ref{fig:fisher_numinamath_qwen}(g) matches the dominant layer observed on T\"{u}lu V3, while the other projections in Figure~\ref{fig:fisher_numinamath_qwen}(a)--(f) have much smaller PAGE values. This indicates that the same Qwen3-8B module is selected even when the probe data shifts from instruction-following to mathematical reasoning.

\begin{figure}[htbp]
    \centering
    \includegraphics[width=1.0\textwidth]{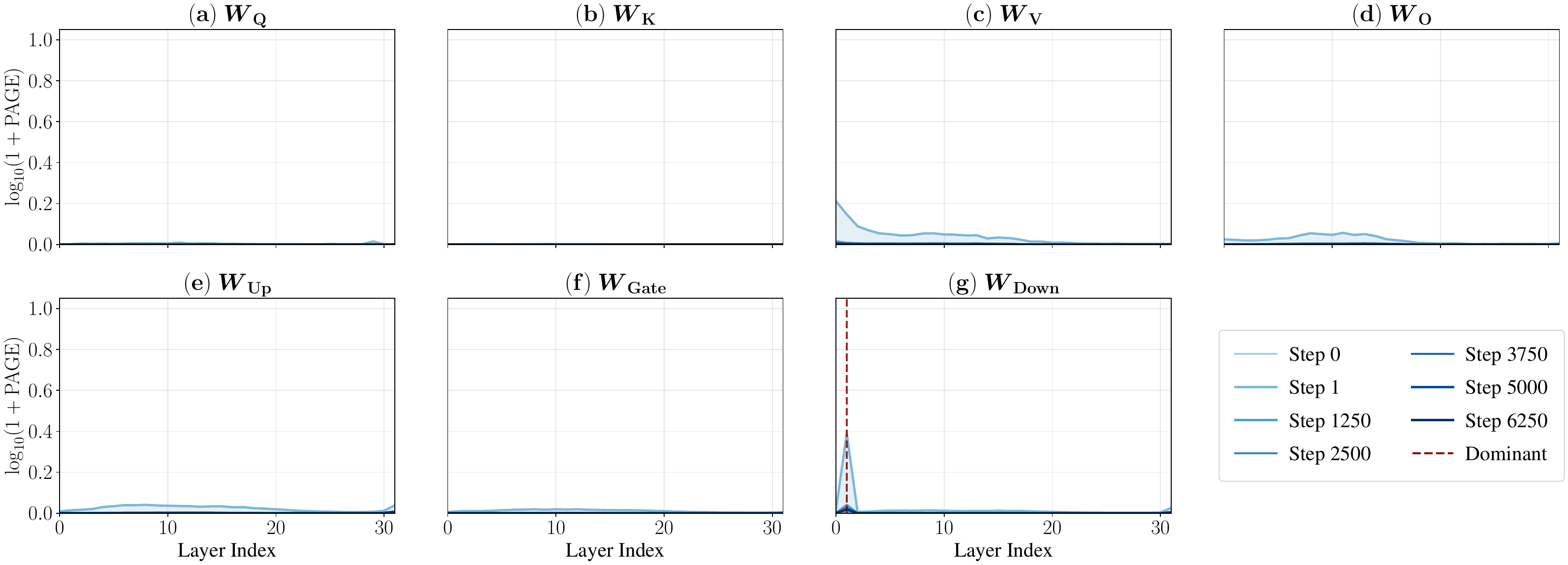}
    \caption{PAGE of all projection modules in LLaMA-3.1-8B-Instruct on NuminaMath-CoT.}
    \label{fig:fisher_numinamath_llama3}
\end{figure}

Figure~\ref{fig:fisher_numinamath_llama3} reports that LLaMA-3.1-8B-Instruct again selects the Layer~1 FFN down-projection on NuminaMath-CoT. Compared with the T\"{u}lu V3 setting, the non-FFN projections in Figure~\ref{fig:fisher_numinamath_llama3}(a)--(f) are weaker and smoother, making the down-projection peak in Figure~\ref{fig:fisher_numinamath_llama3}(g) more isolated. The repeated selection of Layer~1 suggests that the dominant module is stable for this backbone rather than tied to a particular data distribution.

\begin{figure}[htbp]
    \centering
    \includegraphics[width=1.0\textwidth]{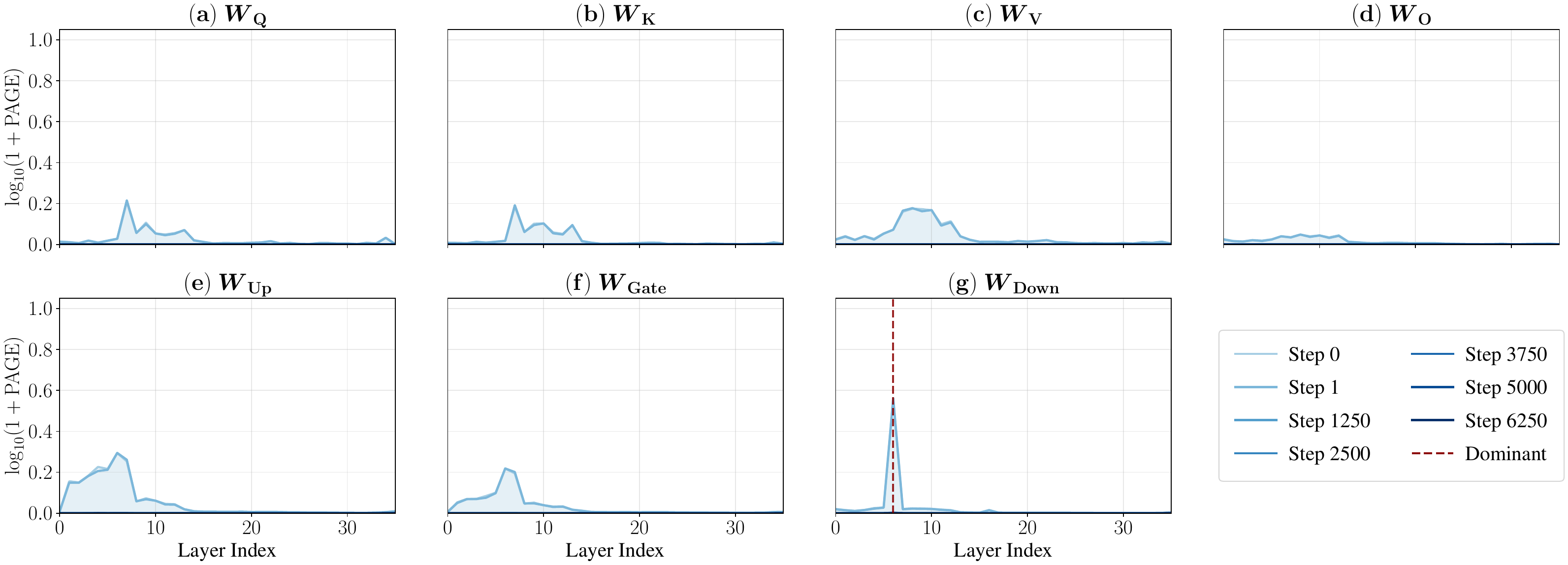}
    \caption{PAGE of all projection modules in Qwen3-8B on KodCode.}
    \label{fig:fisher_Kodocde_qwen}
\end{figure}

Figure~\ref{fig:fisher_Kodocde_qwen} shows that Qwen3-8B on KodCode still has its largest PAGE peak at the Layer~6 FFN down-projection. In Figure~\ref{fig:fisher_Kodocde_qwen}(g), the peak is narrow and clearly separated from the rest of the layer-wise curve. By contrast, the attention projections and FFN up/gate projections in Figure~\ref{fig:fisher_Kodocde_qwen}(a)--(f) do not form a comparable localized maximum. Thus, the same dominant module is recovered on code-generation data.

\begin{figure}[htbp]
    \centering
    \includegraphics[width=1.0\textwidth]{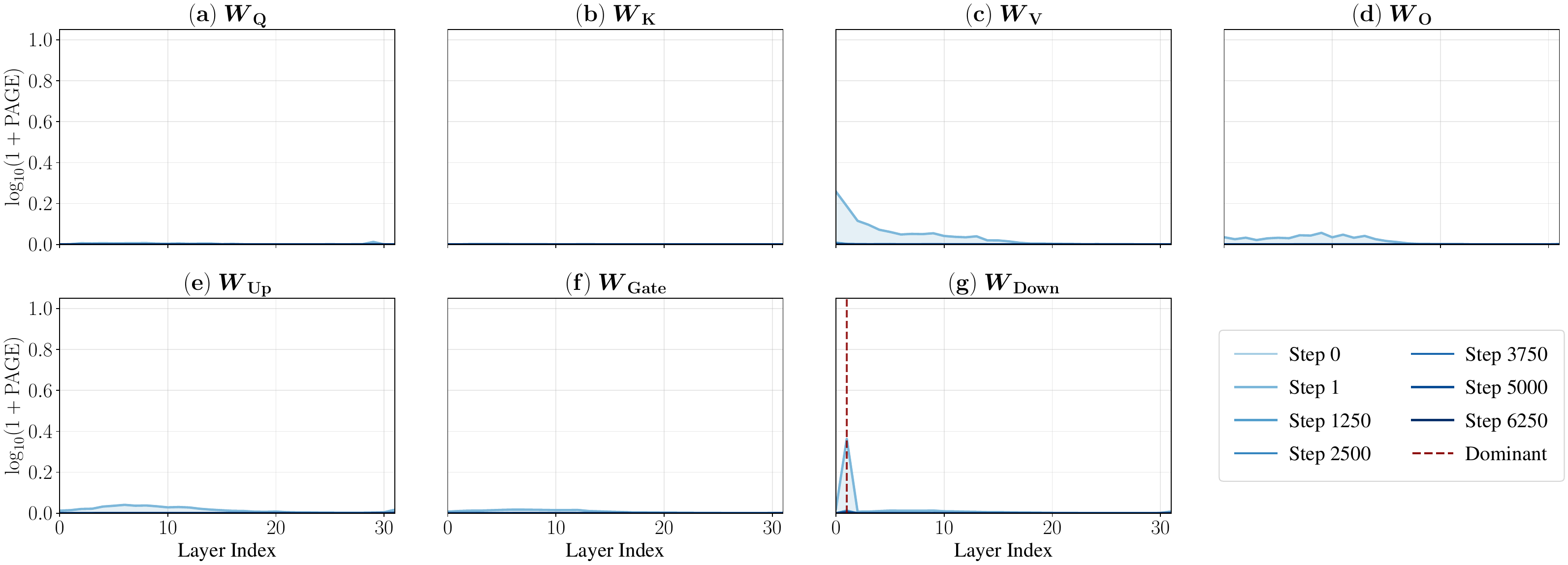}
    \caption{PAGE of all projection modules in LLaMA-3.1-8B-Instruct on KodCode.}
    \label{fig:fisher_Kodocde_llama3}
\end{figure}

Figure~\ref{fig:fisher_Kodocde_llama3} reveals that LLaMA-3.1-8B-Instruct on KodCode again concentrates PAGE at the Layer~1 FFN down-projection. Although shallow attention projections in Figure~\ref{fig:fisher_Kodocde_llama3}(c)--(d) exhibit elevated values at early checkpoints, these signals are broader and decay over training. In contrast, the $W_{\mathrm{down}}$ peak in Figure~\ref{fig:fisher_Kodocde_llama3}(g) remains the most localized and stable signal, making it the consistent placement target for this backbone.

\begin{figure}[htbp]
    \centering
    \includegraphics[width=1.0\textwidth]{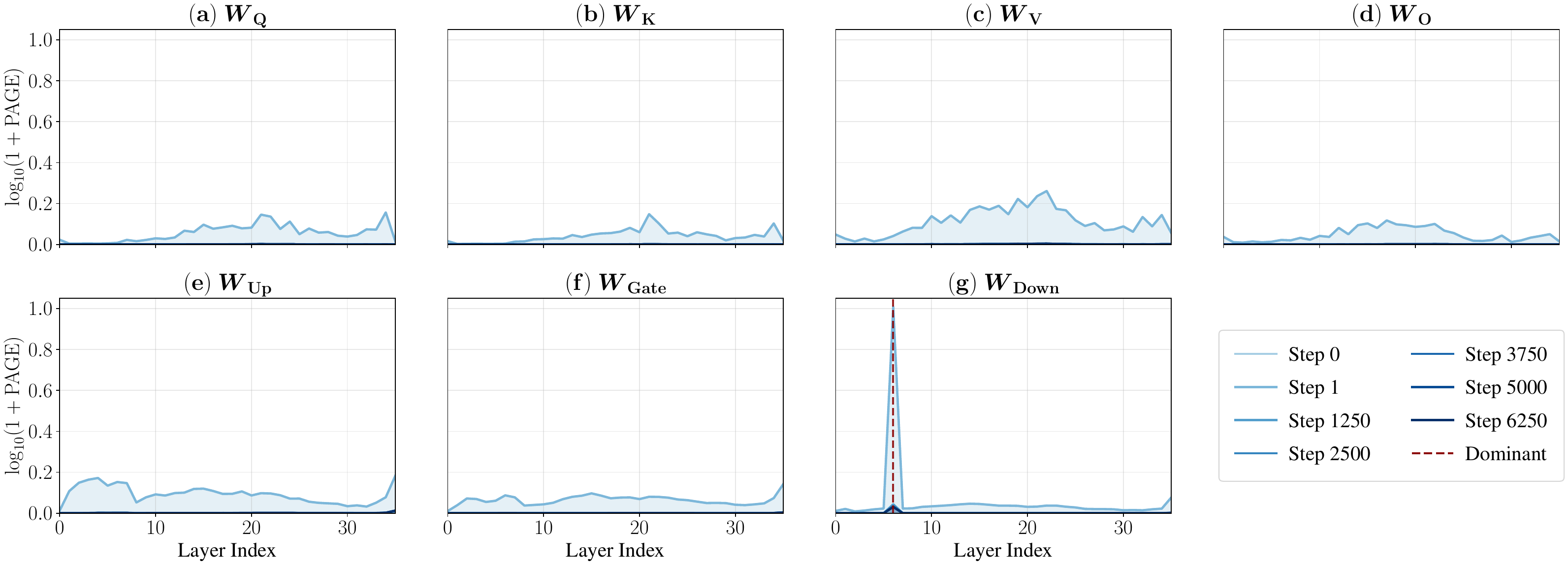}
    \caption{PAGE of all projection modules in Qwen3-8B on WizardLM.}
    \label{fig:fisher_wizardlm_qwen}
\end{figure}

Figure~\ref{fig:fisher_wizardlm_qwen} demonstrates that the strongest PAGE peak for Qwen3-8B on WizardLM again appears at the Layer~6 FFN down-projection. The $W_{\mathrm{down}}$ panel in Figure~\ref{fig:fisher_wizardlm_qwen}(g) contains a clear localized peak, while the remaining projections in Figure~\ref{fig:fisher_wizardlm_qwen}(a)--(f) show only moderate and smooth layer-wise variations. This confirms that Qwen3-8B consistently selects the same dominant module across the evaluated datasets.

\begin{figure}[htbp]
    \centering
    \includegraphics[width=1.0\textwidth]{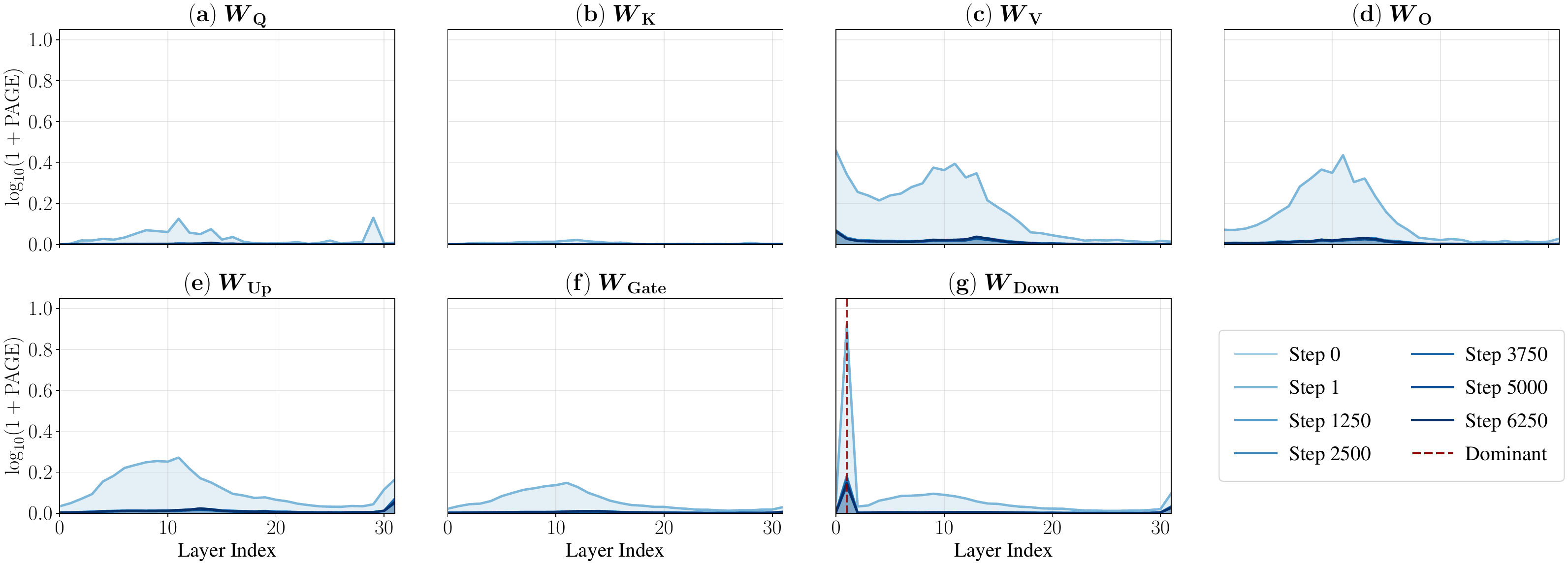}
    \caption{PAGE of all projection modules in LLaMA-3.1-8B-Instruct on WizardLM.}
    \label{fig:fisher_wizardlm_llama3}
\end{figure}

Figure~\ref{fig:fisher_wizardlm_llama3} reports that PAGE for LLaMA-3.1-8B-Instruct on WizardLM is again concentrated at the Layer~1 FFN down-projection. The attention and FFN up/gate projections in Figure~\ref{fig:fisher_wizardlm_llama3}(a)--(f) show only small secondary variations and do not approach the peak in Figure~\ref{fig:fisher_wizardlm_llama3}(g). Together with the other LLaMA figures, this indicates that the dominant module depends on the backbone but remains stable across downstream datasets.

\section{Why PAGE Peaks at a Shallow FFN Down-Projection}
\label{app:mechanism}

Section~\ref{sec:observation} identified the dominant module as a shallow FFN down-projection, and this section examines two measured factors associated with the observed peak at Step~$0$. First, the input to the down-projection concentrates an exceptional amount of energy in a single channel. Second, the backward error signal decays sharply with depth, so it remains large only in shallow layers. We then verify that the ranking reflects which channels the sensitivity flows through, not the raw magnitude of these channels.

\paragraph{The dominant module's input is concentrated in a single channel.} A projection's per-sample weight gradient is an outer product summed over positions, $\bm{G}=\bm{\Delta}^{\top}\bm{X}$, so its energy decomposes exactly across input channels, $\|\bm{G}\|_{\mathrm F}^2=\sum_j\|[\bm{G}]_{:,j}\|_2^2$. By Cauchy--Schwarz, each channel satisfies $\|[\bm{G}]_{:,j}\|_2^2\le D\,\|[\bm{X}]_{:,j}\|_2^2$, where $D=\|\bm{\Delta}\|_{\mathrm F}^2$ is the output-gradient energy. Thus, for a fixed output-gradient energy, low input energy limits a channel's possible contribution to the weight gradient. We compare input-energy and gradient-energy concentration directly in the measured modules, without using PAGE.

The dominant module is extreme on both measures. A single input channel holds $92\%$ (Qwen3-8B, $\bm{W}_{\mathrm{down}}$@6) and $93\%$ (LLaMA-3.1-8B-Instruct, $\bm{W}_{\mathrm{down}}$@1) of the input energy, and its top $32$ channels hold $99.9\%$ and $97.3\%$, respectively (Figure~\ref{fig:app_page_conc}b). Across layers, the down-projection type reaches the largest single-channel input energy of any projection type---up to $250\times$ (Qwen3-8B) and $27\times$ (LLaMA-3.1-8B-Instruct) the strongest competitor (Figure~\ref{fig:app_page_conc}a). At the dominant module, the gradient concentrates on the very channel that carries this forward outlier: its single largest input channel also carries $91\%$ (Qwen3-8B) and $97\%$ (LLaMA-3.1-8B-Instruct) of the module's gradient energy, far above typical modules (for example, $37\%$ for $\bm{W}_{\mathrm v}$@22 in Qwen3-8B and $30\%$ for $\bm{W}_{\mathrm v}$@0 in LLaMA-3.1-8B-Instruct). This structure likely follows from the down-projection's role: it is the projection that reads the gated SwiGLU product and writes directly to the residual stream. Prior work has documented large-magnitude activation and weight outliers at early down-projection inputs \citep{sun2024massive,an2025systematic,yu2025superweight}, but only as a forward-pass or inference phenomenon. We connect this location to fine-tuning: adaptation sensitivity concentrates on the very channel that carries the forward outlier, which makes it a single dominant module for LoRA placement.

\begin{figure}[ht]
    \centering
    \includegraphics[width=\linewidth]{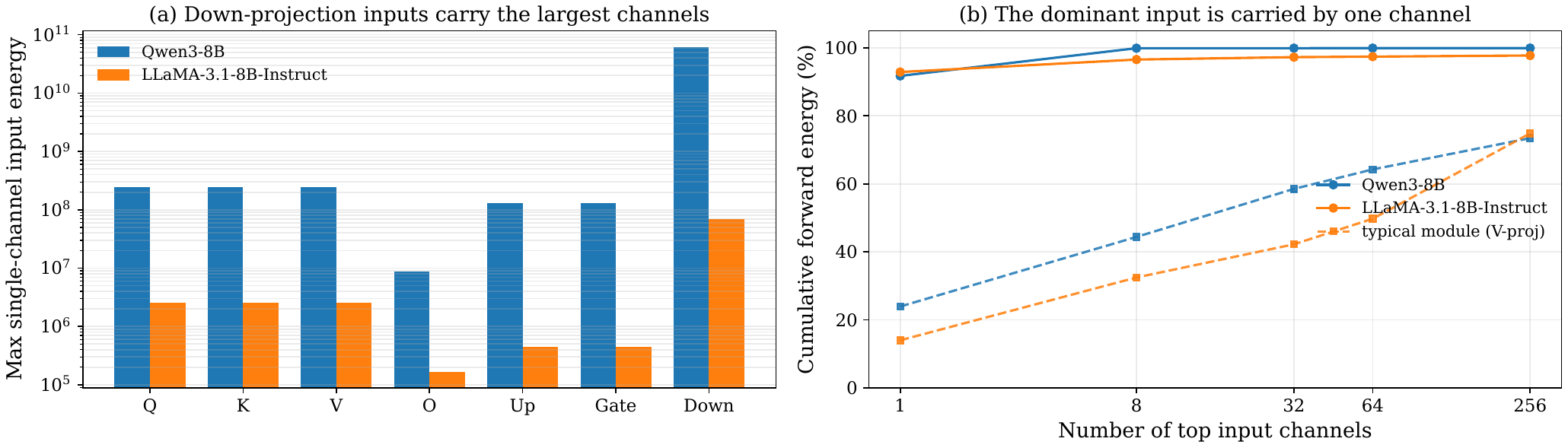}
    \caption{\textbf{Down-projection inputs carry exceptional single-channel energy.}
    (a) Largest single-channel input energy across layers, per module type (log scale); the down-projection type exceeds the strongest competing type by $250\times$ (Qwen3-8B) and $27\times$ (LLaMA-3.1-8B-Instruct).
    (b) Cumulative forward energy versus number of top input channels, for the dominant down-projection (solid) and a typical V-projection (dashed).}
    \label{fig:app_page_conc}
\end{figure}

\paragraph{The backward error signal is large only at shallow depth.} We measure the residual-stream error signal $E_l=\mathbb{E}_i\|\partial\bar{\ell}_i/\partial\bm{r}_l\|_2^2$ at each layer, again at Step~$0$. For the down-projection this is exactly the output-gradient energy $D$ from above, because the down-projection writes straight into the residual stream (other projections differ by intervening Jacobians). From the shallowest to the deepest layer, the signal drops by $3{,}735\times$ in Qwen3-8B and $940\times$ in LLaMA-3.1-8B-Instruct (Figure~\ref{fig:app_signal_decay}). A gradient energy such as PAGE necessarily contains this factor; what is empirical here is that the factor stays large only at shallow depth, matching reports of vanishing gradients in deep Pre-LN Transformers \citep{sun2025curse,li2025mixln}.

\begin{figure}[ht]
    \centering
    \includegraphics[width=\linewidth]{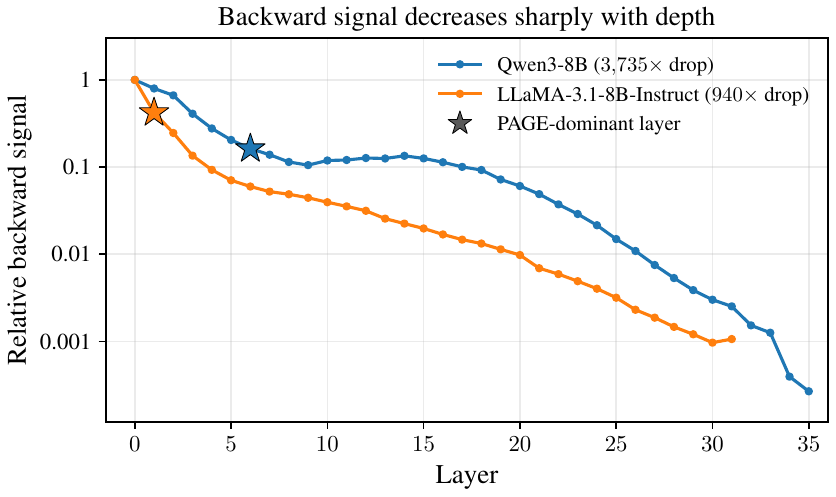}
    \caption{\textbf{The backward error signal decays with depth.} Relative residual-stream backward signal $E_l/E_0$ per layer (pretrained backbones, Step~$0$); the total drop from the shallowest to the deepest layer is given in the legend. Stars mark the PAGE-dominant layers---layer~$6$ (Qwen3-8B) and layer~$1$ (LLaMA-3.1-8B-Instruct)---each shallow enough to still receive a large backward signal.}
    \label{fig:app_signal_decay}
\end{figure}

The two properties select the dominant module jointly, and the near-misses show that neither suffices alone. The deep $\bm{W}_{\mathrm{down}}$@35 in Qwen3-8B has a comparably concentrated input (top-$32$ forward energy $0.97$, against $1.00$ at layer~$6$) and its largest channel carries about $17\times$ more input energy, but its backward signal is about $600\times$ weaker and its gradient no longer follows the dominant input channel, leaving its PAGE roughly $38\times$ below layer~$6$'s---a small bump, not the peak. At the other end, $\bm{W}_{\mathrm{down}}$@0 sits at the largest backward signal with a comparably concentrated input (top-$32$ energy $0.96$), yet its dominant channel carries $274\times$ less energy than layer~$6$'s and its PAGE is $335\times$ lower. The up/gate projections at layer~$6$ fail on the input side instead: they share the depth, but their largest channel carries $88\times$ less energy. The same joint requirement holds in LLaMA-3.1-8B-Instruct.

\paragraph{PAGE ranking under channel attenuation.} The dominance does not come from a few large-magnitude channels. For the mild attenuation settings tested, scaling the $32$ dominant input channels by $\alpha$ preserves the top rank with small changes in probe loss. Stronger attenuation increases the loss and can change the ranking, especially for LLaMA-3.1-8B-Instruct. The rank moves only when the loss does ($\alpha\le0.1$ for LLaMA-3.1-8B-Instruct, $\alpha=0$ for Qwen3-8B), that is, only once attenuation is strong enough to damage the model---at which point the ranking change conflates attenuation with damage, since removal also reroutes gradient to the block's up and gate projections. Scaling $32$ random channels, or another module's top channels, leaves the ranking unchanged (Table~\ref{tab:app_scale}).

\begin{table}[ht]
\centering
\caption{Channel-scaling test: PAGE rank of the dominant module and the probe-set loss (mean cross-entropy on the $32$ probe examples) as its $32$ largest input channels are scaled by $\alpha$. The rank stays at $1$ for the mild settings tested; stronger attenuation increases the loss and eventually changes the ranking, with full removal ($\alpha=0$) producing the strongest change.}
\label{tab:app_scale}
\small
\begin{tblr}{
  width      = 0.62\linewidth,
  colspec    = {Q[l,m,2.0] Q[c,m,1.1] Q[c,m,1.0] Q[c,m,1.0]},
  colsep     = 3pt,
  row{1}     = {font=\bfseries},
  hline{1,Z} = {wd=0.08em},
  hline{2}   = {wd=0.05em},
  hline{8}   = {wd=0.03em},
}
  Backbone & $\alpha$ & PAGE rank & Probe loss \\
  Qwen3-8B & 1.0 & 1 & 1.549 \\
   & 0.5 & 1 & 1.544 \\
   & 0.25 & 1 & 1.541 \\
   & 0.1 & 1 & 1.531 \\
   & 0.05 & 1 & 1.536 \\
   & 0 (removed) & 7 & 1.981 \\
  LLaMA-3.1-8B-Instruct & 1.0 & 1 & 0.984 \\
   & 0.5 & 1 & 0.981 \\
   & 0.25 & 1 & 0.965 \\
   & 0.1 & 4 & 1.150 \\
   & 0.05 & 4 & 3.137 \\
   & 0 (removed) & 37 & 3.534 \\
\end{tblr}
\end{table}

In short, the observed PAGE peak coincides with two measured properties: highly concentrated input-channel sensitivity at the down-projection and a strong residual-stream error signal at shallow depth. We do not explain why the dominant adaptation module sits at Layer~6 in Qwen3-8B and Layer~1 in LLaMA-3.1-8B-Instruct; this is a property of the pretrained model that we measure but do not derive.

\section{Detailed Derivations for PAGE}
\label{app:page_derivations}

This appendix provides the supporting derivations for Section~\ref{sec:emp_fisher}.
We first give an empirical Fisher interpretation of the module sensitivity.
We then prove Lemma~\ref{lem:initial_lora_grads}.
Finally, we derive the PAGE expression in Theorem~\ref{thm:page_closed_form}, including the moment identity in Eq.~\eqref{eq:expected_ATA} and the relationship between PAGE and module sensitivity.

Throughout this appendix, for matrices
\(\bm{X},\bm{Y}\in\mathbb{R}^{m\times n}\), we use the Frobenius inner product
\[
    \left\langle \bm{X},\bm{Y}\right\rangle_{\mathrm{F}}
    =
    \operatorname{tr}\!\left(\bm{X}^{\top}\bm{Y}\right),
\]
and for a matrix \(\bm{M}\in\mathbb{R}^{m\times n}\), we use the Frobenius norm
\[
    \left\|\bm{M}\right\|_{\mathrm{F}}^{2}
    =
    \operatorname{tr}\!\left(\bm{M}^{\top}\bm{M}\right)
    =
    \sum_{a=1}^{m}\sum_{b=1}^{n} M_{ab}^{2}.
\]
All gradients are evaluated at the pretrained model before LoRA optimization
begins. Unless otherwise specified, expectations over \(\bm{A}_{l,k}\) are
taken with respect to its random initialization.

\subsection{Empirical Fisher Interpretation of Module Sensitivity}
\label{app:emp_fisher_trace}

In Definition~\ref{def:emp_fisher}, the module sensitivity
\(\mathcal{S}^{\mathrm{emp}}_{l,k}\) is defined in
Eq.~\eqref{eq:sample_fisher_score} as the average squared Frobenius norm of the
sample-wise gradients in Eq.~\eqref{eq:sample_gradient}.
We now show that this quantity is exactly the trace of the empirical Fisher
block for the pretrained projection weight \(\bm{W}_{l,k}\).

Recall from Eq.~\eqref{eq:sample_gradient} that, for sample \(i\) and module
\((l,k)\), the sample-wise gradient with respect to the pretrained projection
weight is
\begin{equation}
  \bm{G}_{i,l,k}
  =
  \nabla_{\bm{W}_{l,k}}\bar{\ell}_i(\theta),
  \qquad
  \bm{W}_{l,k}\in
  \mathbb{R}^{d_{\mathrm{out},k}\times d_{\mathrm{in},k}}.
  \label{eq:app_sample_gradient_recall}
\end{equation}
Here, \(l\) indexes the Transformer layer, \(k\) indexes the projection type,
and \(\bm{G}_{i,l,k}\) has the same shape as \(\bm{W}_{l,k}\).
Eq.~\eqref{eq:app_sample_gradient_recall} gives the first-order response of
the sample loss \(\bar{\ell}_i(\theta)\) to perturbations of this projection
weight.

To write the empirical Fisher block, we need a parameter vector rather than a
matrix-shaped gradient. We therefore vectorize the matrix gradient in
Eq.~\eqref{eq:app_sample_gradient_recall}. Let \(\operatorname{vec}(\cdot)\)
denote column-wise vectorization. We define
\begin{equation}
    \bm{g}_{i,l,k}
    =
    \operatorname{vec}(\bm{G}_{i,l,k})
    \in
    \mathbb{R}^{d_{\mathrm{out},k}d_{\mathrm{in},k}}.
    \label{eq:app_vectorized_gradient}
\end{equation}
This vector contains all entries of \(\bm{G}_{i,l,k}\) and represents the
sample-wise gradient in the parameter space of \(\bm{W}_{l,k}\).

Using the vectorized gradients in Eq.~\eqref{eq:app_vectorized_gradient}, we
define the empirical Fisher block for module \((l,k)\) as
\begin{equation}
  \bm{F}^{\mathrm{emp}}_{l,k}
  =
  \frac{1}{N}
  \sum_{i=1}^{N}
  \bm{g}_{i,l,k}\bm{g}_{i,l,k}^{\top}.
  \label{eq:app_emp_fisher_block}
\end{equation}
Eq.~\eqref{eq:app_emp_fisher_block} is the empirical Fisher block associated
with the pretrained projection weight \(\bm{W}_{l,k}\). It is the average
outer product of sample-wise gradients and records the gradient energy along
directions in the parameter space of this projection weight.

To compare different modules with a scalar quantity, we take the trace of
Eq.~\eqref{eq:app_emp_fisher_block}. This gives
\begin{equation}
\begin{aligned}
  \operatorname{tr}\!\left(\bm{F}^{\mathrm{emp}}_{l,k}\right)
  &=
  \operatorname{tr}\!\left(
      \frac{1}{N}
      \sum_{i=1}^{N}
      \bm{g}_{i,l,k}\bm{g}_{i,l,k}^{\top}
  \right) \\
  &=
  \frac{1}{N}
  \sum_{i=1}^{N}
  \operatorname{tr}\!\left(
      \bm{g}_{i,l,k}\bm{g}_{i,l,k}^{\top}
  \right).
\end{aligned}
\label{eq:app_fisher_trace_expand}
\end{equation}
The second equality in Eq.~\eqref{eq:app_fisher_trace_expand} uses the
linearity of the trace.

For each sample-wise gradient vector, we simplify the remaining trace term.
Using \(\operatorname{tr}(\bm{a}\bm{a}^{\top})=\bm{a}^{\top}\bm{a}\), we obtain
\begin{equation}
  \operatorname{tr}\!\left(
      \bm{g}_{i,l,k}\bm{g}_{i,l,k}^{\top}
  \right)
  =
  \bm{g}_{i,l,k}^{\top}\bm{g}_{i,l,k}
  =
  \left\|
      \bm{g}_{i,l,k}
  \right\|_2^2.
  \label{eq:app_outer_trace_to_norm}
\end{equation}
Eq.~\eqref{eq:app_outer_trace_to_norm} turns the trace of each outer product
into the squared Euclidean norm of the corresponding sample-wise gradient
vector.

Substituting Eq.~\eqref{eq:app_outer_trace_to_norm} into
Eq.~\eqref{eq:app_fisher_trace_expand}, we get
\begin{equation}
  \operatorname{tr}\!\left(\bm{F}^{\mathrm{emp}}_{l,k}\right)
  =
  \frac{1}{N}
  \sum_{i=1}^{N}
  \left\|
      \bm{g}_{i,l,k}
  \right\|_2^2.
  \label{eq:app_fisher_trace_vector_norm}
\end{equation}
Eq.~\eqref{eq:app_fisher_trace_vector_norm} shows that the Fisher trace is the
average squared Euclidean norm of the vectorized sample-wise gradients.

Finally, we substitute
\(\bm{g}_{i,l,k}=\operatorname{vec}(\bm{G}_{i,l,k})\) from
Eq.~\eqref{eq:app_vectorized_gradient}. Since vectorization preserves the squared norm of a matrix, we obtain
\begin{equation}
    \left\|
        \bm{g}_{i,l,k}
    \right\|_2^2
    =
    \left\|
        \operatorname{vec}(\bm{G}_{i,l,k})
    \right\|_2^2
    =
    \left\|
        \bm{G}_{i,l,k}
    \right\|_{\mathrm{F}}^2.
    \label{eq:app_vec_preserves_norm}
\end{equation}
Applying Eq.~\eqref{eq:app_vec_preserves_norm} to every term in
Eq.~\eqref{eq:app_fisher_trace_vector_norm}, we obtain

\begin{equation}
\begin{aligned}
  \operatorname{tr}\!\left(\bm{F}^{\mathrm{emp}}_{l,k}\right)
  &=
  \frac{1}{N}
  \sum_{i=1}^{N}
  \left\|
      \bm{G}_{i,l,k}
  \right\|_{\mathrm{F}}^{2} \\
  &=
  \mathcal{S}^{\mathrm{emp}}_{l,k}.
\end{aligned}
\label{eq:app_fisher_trace_equals_sensitivity}
\end{equation}
The last equality follows from Eq.~\eqref{eq:sample_fisher_score} in
Definition~\ref{def:emp_fisher}. Thus, the module sensitivity
\(\mathcal{S}^{\mathrm{emp}}_{l,k}\) is exactly the trace of the empirical
Fisher block for \(\bm{W}_{l,k}\). This explains why the average squared sample-wise gradient norm in Definition~\ref{def:emp_fisher} can be interpreted as empirical Fisher sensitivity.

\subsection{Proof of Lemma~\ref{lem:initial_lora_grads}
            (Initial LoRA Gradients)}
\label{app:initial_lora_gradient}

\begin{lemma*}
Under the standard LoRA initialization
\(\bm{B}_{l,k}=\bm{0}\), with \(\bm{A}_{l,k}\) initialized by the
Kaiming-uniform rule, the initial trainable gradients of the LoRA factors are
\begin{equation}
    \nabla_{\bm{A}_{l,k}}\bar{\ell}_i = \bm{0},
    \qquad
    \nabla_{\bm{B}_{l,k}}\bar{\ell}_i
    =
    s\,\bm{G}_{i,l,k}\,\bm{A}_{l,k}^{\top},
    \label{eq:app_initial_lora_grads_statement}
\end{equation}
where \(\bm{G}_{i,l,k}\) is the sample-wise full-weight gradient defined in
Eq.~\eqref{eq:sample_gradient}.
\end{lemma*}

\begin{proof}
We prove the lemma in two steps. First, we derive the gradients with respect to the LoRA factors for arbitrary \(\bm{A}_{l,k}\) and \(\bm{B}_{l,k}\). Second, we substitute the standard LoRA initialization \(\bm{B}_{l,k}=\bm{0}\).

\paragraph{Step 1: Derive the general LoRA-factor gradients.}
Recall from Eq.~\eqref{eq:lora_reparam} that, for module \((l,k)\), the LoRA-augmented effective weight is
\begin{equation}
    \bm{W}_{l,k}
    =
    \bm{W}^{0}_{l,k}
    +
    s\,\bm{B}_{l,k}\bm{A}_{l,k},
    \qquad
    s=\alpha/r.
    \label{eq:app_lora_reparam}
\end{equation}
Here, \(\bm{W}^{0}_{l,k}\in\mathbb{R}^{d_{\mathrm{out},k}\times d_{\mathrm{in},k}}\)
is the frozen pretrained projection weight,
\(\bm{B}_{l,k}\in\mathbb{R}^{d_{\mathrm{out},k}\times r}\) and
\(\bm{A}_{l,k}\in\mathbb{R}^{r\times d_{\mathrm{in},k}}\) are trainable LoRA
factors, \(r\) is the LoRA rank, and \(s=\alpha/r\) is the LoRA scaling
coefficient. This equation shows that the loss depends on \(\bm{A}_{l,k}\) and
\(\bm{B}_{l,k}\) only through the low-rank update
\(s\bm{B}_{l,k}\bm{A}_{l,k}\).

To compute the gradients of the LoRA factors, we track how a small change in
\(\bm{A}_{l,k}\) or \(\bm{B}_{l,k}\) changes the effective weight
\(\bm{W}_{l,k}\), and then how this change affects the loss. We first take the
differential of Eq.~\eqref{eq:app_lora_reparam}. Since
\(\bm{W}^{0}_{l,k}\) is frozen, its differential is zero. Therefore, we obtain
\begin{equation}
    \mathrm{d}\bm{W}_{l,k}
    =
    s\,\mathrm{d}\bm{B}_{l,k}\,\bm{A}_{l,k}
    +
    s\,\bm{B}_{l,k}\,\mathrm{d}\bm{A}_{l,k}.
    \label{eq:app_lora_weight_differential}
\end{equation}
Eq.~\eqref{eq:app_lora_weight_differential} separates the change in effective weight
 into the part caused by \(\mathrm{d}\bm{B}_{l,k}\) and the part caused
by \(\mathrm{d}\bm{A}_{l,k}\).

We next relate this change to the sample loss. By the
definition of \(\bm{G}_{i,l,k}\) in Eq.~\eqref{eq:sample_gradient}, the
first-order loss differential for sample \(i\) is
\begin{equation}
    \mathrm{d}\bar{\ell}_i
    =
    \left\langle
        \bm{G}_{i,l,k},\;
        \mathrm{d}\bm{W}_{l,k}
    \right\rangle_{\mathrm{F}}.
    \label{eq:app_loss_differential_dense}
\end{equation}
Eq.~\eqref{eq:app_loss_differential_dense} states that, to first order, the
loss change is the Frobenius inner product between the full-weight gradient and
the change in effective weight.

Substituting Eq.~\eqref{eq:app_lora_weight_differential} into
Eq.~\eqref{eq:app_loss_differential_dense}, we get
\begin{equation}
\begin{aligned}
    \mathrm{d}\bar{\ell}_i
    &=
    s
    \left\langle
        \bm{G}_{i,l,k},\;
        \mathrm{d}\bm{B}_{l,k}\,\bm{A}_{l,k}
    \right\rangle_{\mathrm{F}}
    +
    s
    \left\langle
        \bm{G}_{i,l,k},\;
        \bm{B}_{l,k}\,\mathrm{d}\bm{A}_{l,k}
    \right\rangle_{\mathrm{F}}.
\end{aligned}
\label{eq:app_loss_differential_split}
\end{equation}
Eq.~\eqref{eq:app_loss_differential_split} expresses the same loss differential
as the sum of a \(\mathrm{d}\bm{B}_{l,k}\) term and a
\(\mathrm{d}\bm{A}_{l,k}\) term. We next rewrite these two terms so that
\(\mathrm{d}\bm{B}_{l,k}\) and \(\mathrm{d}\bm{A}_{l,k}\) appear as the second
arguments of Frobenius inner products.

For the \(\mathrm{d}\bm{B}_{l,k}\) term, we expand the Frobenius inner product
using its trace form and move factors by cyclic invariance of the trace:
\begin{equation}
\begin{aligned}
    \left\langle
        \bm{G}_{i,l,k},\;
        \mathrm{d}\bm{B}_{l,k}\bm{A}_{l,k}
    \right\rangle_{\mathrm{F}}
    &=
    \operatorname{tr}\!\left(
        \bm{G}_{i,l,k}^{\top}
        \mathrm{d}\bm{B}_{l,k}\bm{A}_{l,k}
    \right) \\
    &=
    \operatorname{tr}\!\left(
        \bm{A}_{l,k}
        \bm{G}_{i,l,k}^{\top}
        \mathrm{d}\bm{B}_{l,k}
    \right) \\
    &=
    \operatorname{tr}\!\left(
        \left(
            \bm{G}_{i,l,k}\bm{A}_{l,k}^{\top}
        \right)^{\top}
        \mathrm{d}\bm{B}_{l,k}
    \right) \\
    &=
    \left\langle
        \bm{G}_{i,l,k}\bm{A}_{l,k}^{\top},\;
        \mathrm{d}\bm{B}_{l,k}
    \right\rangle_{\mathrm{F}}.
\end{aligned}
\label{eq:app_db_term}
\end{equation}
Eq.~\eqref{eq:app_db_term} identifies the coefficient of
\(\mathrm{d}\bm{B}_{l,k}\) as \(\bm{G}_{i,l,k}\bm{A}_{l,k}^{\top}\).

Similarly, for the \(\mathrm{d}\bm{A}_{l,k}\) term, we again use the trace form
of the Frobenius inner product:
\begin{equation}
\begin{aligned}
    \left\langle
        \bm{G}_{i,l,k},\;
        \bm{B}_{l,k}\mathrm{d}\bm{A}_{l,k}
    \right\rangle_{\mathrm{F}}
    &=
    \operatorname{tr}\!\left(
        \bm{G}_{i,l,k}^{\top}
        \bm{B}_{l,k}\mathrm{d}\bm{A}_{l,k}
    \right) \\
    &=
    \operatorname{tr}\!\left(
        \left(
            \bm{B}_{l,k}^{\top}\bm{G}_{i,l,k}
        \right)^{\top}
        \mathrm{d}\bm{A}_{l,k}
    \right) \\
    &=
    \left\langle
        \bm{B}_{l,k}^{\top}\bm{G}_{i,l,k},\;
        \mathrm{d}\bm{A}_{l,k}
    \right\rangle_{\mathrm{F}}.
\end{aligned}
\label{eq:app_da_term}
\end{equation}
Eq.~\eqref{eq:app_da_term} identifies the coefficient of
\(\mathrm{d}\bm{A}_{l,k}\) as \(\bm{B}_{l,k}^{\top}\bm{G}_{i,l,k}\).

Substituting Eq.~\eqref{eq:app_db_term} and Eq.~\eqref{eq:app_da_term} into
Eq.~\eqref{eq:app_loss_differential_split}, we obtain
\begin{equation}
    \mathrm{d}\bar{\ell}_i
    =
    \left\langle
        s\,\bm{G}_{i,l,k}\bm{A}_{l,k}^{\top},\;
        \mathrm{d}\bm{B}_{l,k}
    \right\rangle_{\mathrm{F}}
    +
    \left\langle
        s\,\bm{B}_{l,k}^{\top}\bm{G}_{i,l,k},\;
        \mathrm{d}\bm{A}_{l,k}
    \right\rangle_{\mathrm{F}}.
\label{eq:app_lora_factor_differential}
\end{equation}
This equation expresses the loss differential directly in terms of the
perturbations of the two LoRA factors.

By the definition of gradients with respect to \(\bm{B}_{l,k}\) and
\(\bm{A}_{l,k}\), the same differential can also be written as
\begin{equation}
    \mathrm{d}\bar{\ell}_i
    =
    \left\langle
        \nabla_{\bm{B}_{l,k}}\bar{\ell}_i,\;
        \mathrm{d}\bm{B}_{l,k}
    \right\rangle_{\mathrm{F}}
    +
    \left\langle
        \nabla_{\bm{A}_{l,k}}\bar{\ell}_i,\;
        \mathrm{d}\bm{A}_{l,k}
    \right\rangle_{\mathrm{F}}.
    \label{eq:app_lora_gradient_definition}
\end{equation}
Comparing Eq.~\eqref{eq:app_lora_factor_differential} with
Eq.~\eqref{eq:app_lora_gradient_definition}, we obtain the general LoRA factor
gradients:
\begin{equation}
    \nabla_{\bm{B}_{l,k}}\bar{\ell}_i
    =
    s\,\bm{G}_{i,l,k}\bm{A}_{l,k}^{\top},
    \qquad
    \nabla_{\bm{A}_{l,k}}\bar{\ell}_i
    =
    s\,\bm{B}_{l,k}^{\top}\bm{G}_{i,l,k}.
    \label{eq:app_general_lora_factor_gradients}
\end{equation}
Eq.~\eqref{eq:app_general_lora_factor_gradients} gives the gradients for
arbitrary values of the LoRA factors.

\paragraph{Step 2: Apply the standard LoRA initialization.}
We now specialize Eq.~\eqref{eq:app_general_lora_factor_gradients} to the
standard LoRA initialization. At initialization, \(\bm{B}_{l,k}=\bm{0}\), while
\(\bm{A}_{l,k}\) is a Kaiming-uniform initialized matrix. Substituting
\(\bm{B}_{l,k}=\bm{0}\) into Eq.~\eqref{eq:app_general_lora_factor_gradients}
gives
\begin{equation}
    \nabla_{\bm{A}_{l,k}}\bar{\ell}_i
    =
    s\,\bm{0}^{\top}\bm{G}_{i,l,k}
    =
    \bm{0},
    \qquad
    \nabla_{\bm{B}_{l,k}}\bar{\ell}_i
    =
    s\,\bm{G}_{i,l,k}\bm{A}_{l,k}^{\top}.
    \label{eq:app_initial_lora_grads}
\end{equation}
Thus, only \(\bm{B}_{l,k}\) receives a nonzero gradient at initialization, and
this gradient is the full-weight gradient \(\bm{G}_{i,l,k}\) projected through
the initialized factor \(\bm{A}_{l,k}\). This recovers
Eq.~\eqref{eq:initial_lora_grads} and proves Lemma~\ref{lem:initial_lora_grads}.
\end{proof}

\subsection{Derivation of the PAGE Expression}
\label{app:page_expression}
\begin{theorem*}
\label{thm:app_page_closed_form}
Let the entries of $\bm{A}_{l,k}\in\mathbb{R}^{r\times d_{\mathrm{in},k}}$
be drawn i.i.d.\ from
$\mathcal{U}\!\bigl(
    -1/\sqrt{d_{\mathrm{in},k}},\;
     1/\sqrt{d_{\mathrm{in},k}}
\bigr)$.
Then
\begin{equation}
    \mathbb{E}_{\bm{A}_{l,k}}
    \!\left[
        \bm{A}_{l,k}^{\top}\bm{A}_{l,k}
    \right]
    =
    \frac{r}{3\,d_{\mathrm{in},k}}\,
    \bm{I}_{d_{\mathrm{in},k}},
    \label{eq:app_expected_ATA}
\end{equation}
where $\bm{I}_{d_{\mathrm{in},k}}$ is a
$d_{\mathrm{in},k}\times d_{\mathrm{in},k}$ identity matrix.
Consequently, PAGE can be written as:
\begin{equation}
    \mathrm{PAGE}_{l,k}
    =
    \frac{s^{2}\,r}{3\,d_{\mathrm{in},k}}\;
    \mathcal{S}^{\mathrm{emp}}_{l,k}.
    \label{eq:app_page_closed_form_statement}
\end{equation}
\end{theorem*}

\begin{proof}
We prove Theorem~\ref{thm:page_closed_form} in two steps. First, we compute
the initialization moment
\(\mathbb{E}_{\bm{A}_{l,k}}[\bm{A}_{l,k}^{\top}\bm{A}_{l,k}]\), which proves
Eq.~\eqref{eq:app_expected_ATA}. Second, we substitute this moment into the trace
form of PAGE in Eq.~\eqref{eq:page_trace_form} to derive
Eq.~\eqref{eq:app_page_closed_form_statement}.

\paragraph{Step 1: Compute the initialization moment.}
We first prove Eq.~\eqref{eq:app_expected_ATA}. Let
\(\bm{A}_{l,k}\in\mathbb{R}^{r\times d_{\mathrm{in},k}}\) be the LoRA factor
in Theorem~\ref{thm:page_closed_form}. For clarity, let
\[
    A_{qj}
    =
    (\bm{A}_{l,k})_{qj},
    \qquad
    1\le q\le r,\quad
    1\le j\le d_{\mathrm{in},k}.
\]
Here, \(q\) indexes the rank dimension of the LoRA factor, and \(j\) indexes
the input dimension of the projection weight. Under the initialization
assumption in Theorem~\ref{thm:page_closed_form}, the entries \(A_{qj}\) are
independent and identically distributed as
\begin{equation}
    A_{qj}
    \sim
    \mathcal{U}\!\left(
        -\frac{1}{\sqrt{d_{\mathrm{in},k}}},
         \frac{1}{\sqrt{d_{\mathrm{in},k}}}
    \right).
    \label{eq:app_A_entry_distribution}
\end{equation}
From the uniform distribution in Eq.~\eqref{eq:app_A_entry_distribution}, we
compute the mean and second moment of each entry as
\begin{equation}
    \mathbb{E}_{\bm{A}_{l,k}}[A_{qj}]
    =
    0,
    \qquad
    \mathbb{E}_{\bm{A}_{l,k}}[A_{qj}^{2}]
    =
    \frac{1}{3\,d_{\mathrm{in},k}}.
    \label{eq:app_A_entry_moments}
\end{equation}
Eq.~\eqref{eq:app_A_entry_moments} provides the entry-wise moments needed for
the matrix expectation. We compute this expectation by evaluating the
\((j,j')\)-th entry of \(\bm{A}_{l,k}^{\top}\bm{A}_{l,k}\):

\begin{equation}
    \left(
        \bm{A}_{l,k}^{\top}\bm{A}_{l,k}
    \right)_{jj'}
    =
    \sum_{q=1}^{r}
    A_{qj}A_{qj'},
    \qquad
    1\le j,j'\le d_{\mathrm{in},k}.
    \label{eq:app_ATA_entry}
\end{equation}
The expectation in Eq.~\eqref{eq:app_ATA_entry} depends on whether the two
column indices \(j\) and \(j'\) are equal. When \(j\neq j'\), the product
\(A_{qj}A_{qj'}\) contains two independent zero-mean entries. When \(j=j'\), it
becomes the squared entry \(A_{qj}^{2}\). We therefore treat the off-diagonal
and diagonal cases separately.

For the off-diagonal case \(j\neq j'\), taking expectation on both sides of
Eq.~\eqref{eq:app_ATA_entry} gives
\begin{equation}
\begin{aligned}
    \mathbb{E}_{\bm{A}_{l,k}}
    \!\left[
        \left(
            \bm{A}_{l,k}^{\top}\bm{A}_{l,k}
        \right)_{jj'}
    \right]
    &=
    \sum_{q=1}^{r}
    \mathbb{E}_{\bm{A}_{l,k}}
    \!\left[
        A_{qj}A_{qj'}
    \right] \\
    &=
    \sum_{q=1}^{r}
    \mathbb{E}_{\bm{A}_{l,k}}[A_{qj}]
    \mathbb{E}_{\bm{A}_{l,k}}[A_{qj'}] \\
    &=
    0.
\end{aligned}
\label{eq:app_ATA_off_diagonal}
\end{equation}
The second equality uses the independence of \(A_{qj}\) and \(A_{qj'}\), and
the last equality uses the zero mean in Eq.~\eqref{eq:app_A_entry_moments}.
Thus, every off-diagonal entry of the expected matrix is zero.

For the diagonal case \(j=j'\), Eq.~\eqref{eq:app_ATA_entry} becomes
\begin{equation}
    \left(
        \bm{A}_{l,k}^{\top}\bm{A}_{l,k}
    \right)_{jj}
    =
    \sum_{q=1}^{r}
    A_{qj}^{2}.
    \label{eq:app_ATA_diagonal_entry}
\end{equation}
Taking expectation on both sides of Eq.~\eqref{eq:app_ATA_diagonal_entry} and
using the second moment in Eq.~\eqref{eq:app_A_entry_moments}, we obtain
\begin{equation}
\begin{aligned}
    \mathbb{E}_{\bm{A}_{l,k}}
    \!\left[
        \left(
            \bm{A}_{l,k}^{\top}\bm{A}_{l,k}
        \right)_{jj}
    \right]
    &=
    \sum_{q=1}^{r}
    \mathbb{E}_{\bm{A}_{l,k}}[A_{qj}^{2}] \\
    &=
    \sum_{q=1}^{r}
    \frac{1}{3\,d_{\mathrm{in},k}} \\
    &=
    \frac{r}{3\,d_{\mathrm{in},k}}.
\end{aligned}
\label{eq:app_ATA_diagonal}
\end{equation}
Thus, every diagonal entry of the expected matrix equals
\(r/(3d_{\mathrm{in},k})\).

The off-diagonal result in Eq.~\eqref{eq:app_ATA_off_diagonal} shows that all
off-diagonal entries are zero, while the diagonal entries result in
Eq.~\eqref{eq:app_ATA_diagonal} shows that every diagonal entry equals
\(r/(3d_{\mathrm{in},k})\). Therefore, the expected matrix is a scalar multiple
of the identity matrix:
\begin{equation}
    \mathbb{E}_{\bm{A}_{l,k}}
    \!\left[
        \bm{A}_{l,k}^{\top}\bm{A}_{l,k}
    \right]
    =
    \frac{r}{3\,d_{\mathrm{in},k}}\,
    \bm{I}_{d_{\mathrm{in},k}}.
    \label{eq:app_expected_ATA1}
\end{equation}
Eq.~\eqref{eq:app_expected_ATA1} proves Eq.~\eqref{eq:expected_ATA} in
Theorem~\ref{thm:page_closed_form}.

\paragraph{Step 2: Substitute the initialization moment into PAGE.}
We now derive Eq.~\eqref{eq:page_closed_form}. Recall from Eq.~\eqref{eq:page_trace_form} that PAGE can be written as
\begin{equation}
    \mathrm{PAGE}_{l,k}
    =
    \frac{s^{2}}{N}
    \sum_{i=1}^{N}
    \mathbb{E}_{\bm{A}_{l,k}}
    \!\left[
        \operatorname{tr}
        \!\left(
            \bm{G}_{i,l,k}^{\top}\bm{G}_{i,l,k}\;
            \bm{A}_{l,k}^{\top}\bm{A}_{l,k}
        \right)
    \right].
    \label{eq:app_page_start}
\end{equation}
Here, \(\bm{G}_{i,l,k}\) is the sample-wise full-weight gradient from
Eq.~\eqref{eq:sample_gradient}, and the expectation is taken over the random
initialization of \(\bm{A}_{l,k}\).

Since \(\bm{G}_{i,l,k}\) is computed on the pretrained model before LoRA
optimization begins, it is fixed with respect to the random initialization of
\(\bm{A}_{l,k}\). Therefore, the expectation in Eq.~\eqref{eq:app_page_start}
acts only on \(\bm{A}_{l,k}^{\top}\bm{A}_{l,k}\). Moving this expectation
inside the trace gives
\begin{equation}
    \mathbb{E}_{\bm{A}_{l,k}}
    \!\left[
        \operatorname{tr}
        \!\left(
            \bm{G}_{i,l,k}^{\top}\bm{G}_{i,l,k}\;
            \bm{A}_{l,k}^{\top}\bm{A}_{l,k}
        \right)
    \right]
    =
    \operatorname{tr}
    \!\left(
        \bm{G}_{i,l,k}^{\top}\bm{G}_{i,l,k}\;
        \mathbb{E}_{\bm{A}_{l,k}}
        \!\left[
            \bm{A}_{l,k}^{\top}\bm{A}_{l,k}
        \right]
    \right).
    \label{eq:app_expectation_inside_trace}
\end{equation}
Eq.~\eqref{eq:app_expectation_inside_trace} isolates the only random term \(\bm{A}_{l,k}^{\top}\bm{A}_{l,k}\)

Substituting Eq.~\eqref{eq:app_expected_ATA1} into
Eq.~\eqref{eq:app_expectation_inside_trace}, we get
\begin{equation}
    \operatorname{tr}
    \!\left(
        \bm{G}_{i,l,k}^{\top}\bm{G}_{i,l,k}\;
        \mathbb{E}_{\bm{A}_{l,k}}
        \!\left[
            \bm{A}_{l,k}^{\top}\bm{A}_{l,k}
        \right]
    \right) 
    =
    \operatorname{tr}
    \!\left(
        \bm{G}_{i,l,k}^{\top}\bm{G}_{i,l,k}\;
        \frac{r}{3\,d_{\mathrm{in},k}}\,
        \bm{I}_{d_{\mathrm{in},k}}
    \right).
\label{eq:app_substitute_expected_ATA}
\end{equation}
Eq.~\eqref{eq:app_substitute_expected_ATA} replaces the random projection term
with a deterministic scalar multiple of the identity matrix.

Because \(r/(3d_{\mathrm{in},k})\) is a scalar and
\(\bm{I}_{d_{\mathrm{in},k}}\) is the identity matrix, the right-hand side of
Eq.~\eqref{eq:app_substitute_expected_ATA} becomes
\begin{equation}
    \operatorname{tr}
    \!\left(
        \bm{G}_{i,l,k}^{\top}\bm{G}_{i,l,k}\;
        \frac{r}{3\,d_{\mathrm{in},k}}\,
        \bm{I}_{d_{\mathrm{in},k}}
    \right)
    =
    \frac{r}{3\,d_{\mathrm{in},k}}
    \operatorname{tr}
    \!\left(
        \bm{G}_{i,l,k}^{\top}\bm{G}_{i,l,k}
    \right)
    =
    \frac{r}{3\,d_{\mathrm{in},k}}
    \left\|
        \bm{G}_{i,l,k}
    \right\|_{\mathrm{F}}^{2}.
\label{eq:app_expected_projected_norm}
\end{equation}
The last equality in Eq.~\eqref{eq:app_expected_projected_norm} uses the
Frobenius norm identity
\(\|\bm{G}_{i,l,k}\|_{\mathrm{F}}^2
=
\operatorname{tr}(\bm{G}_{i,l,k}^{\top}\bm{G}_{i,l,k})\).

Substituting Eq.~\eqref{eq:app_expected_projected_norm} into
Eq.~\eqref{eq:app_page_start}, we obtain
\begin{equation}
    \mathrm{PAGE}_{l,k}
    =
    \frac{s^{2}}{N}
    \sum_{i=1}^{N}
    \frac{r}{3\,d_{\mathrm{in},k}}
    \left\|
        \bm{G}_{i,l,k}
    \right\|_{\mathrm{F}}^{2}.
    \label{eq:app_page_after_projection}
\end{equation}
Eq.~\eqref{eq:app_page_after_projection} shows that PAGE is proportional to the
average squared Frobenius norm of the sample-wise full-weight gradients.

Factoring out the constants \(s^2\), \(r\), and \(3d_{\mathrm{in},k}\) from
Eq.~\eqref{eq:app_page_after_projection}, we obtain
\begin{equation}
    \mathrm{PAGE}_{l,k}
    =
    \frac{s^{2}r}{3\,d_{\mathrm{in},k}}
    \left(
        \frac{1}{N}
        \sum_{i=1}^{N}
        \left\|
            \bm{G}_{i,l,k}
        \right\|_{\mathrm{F}}^{2}
    \right).
    \label{eq:app_page_factorized}
\end{equation}
The term in parentheses is exactly the module sensitivity
\(\mathcal{S}^{\mathrm{emp}}_{l,k}\) defined in Eq.~\eqref{eq:sample_fisher_score}.
Substituting Eq.~\eqref{eq:sample_fisher_score} into
Eq.~\eqref{eq:app_page_factorized}, we get
\begin{equation}
    \mathrm{PAGE}_{l,k}
    =
    \frac{s^{2}r}{3\,d_{\mathrm{in},k}}\;
    \mathcal{S}^{\mathrm{emp}}_{l,k}.
    \label{eq:app_page_closed_form}
\end{equation}
This completes the proof.
\end{proof}

\section{Experimental Details}
\label{app:exp_details}

\paragraph{Hardware and default setup.}
We run training on a single NVIDIA H100 GPU using PyTorch.
Unless otherwise specified, we use the default training configuration in
Table~\ref{tab:default_hyper}.

\begin{table}[ht]
\centering
\caption{Default training configuration.}
\label{tab:default_hyper}
\small
\begin{tblr}{
  width       = 0.6\linewidth,
  colspec     = {
    Q[l,m,1.25]
    Q[l,m,1.55]
  },
  colsep      = 2.5pt,
  rowsep      = 1pt,
  row{1}      = {font=\bfseries},
}
  \toprule
  Setting                       & Value \\
  \midrule
  Precision                     & \texttt{bfloat16}, TF32 enabled \\
  DeepSpeed stage               & ZeRO-2 \\
  Epochs                        & 1 \\
  Per-device batch size         & 4 \\
  Gradient accumulation steps   & 4 \\
  Effective batch size          & 16 \\
  Maximum sequence length       & 3,072 \\
  Learning-rate schedule        & Linear warmup, then cosine decay \\
  Peak learning rate            & $2 \times 10^{-5}$ \\
  Warmup ratio                  & 0.03 of total training steps \\
  Warmup initial learning rate  & 0 \\
  Weight decay                  & 0 \\
  LoRA rank $r$                 & 64 \\
  LoRA $\alpha$                 & 128 \\
  LoRA dropout                  & 0.05 \\
  Trainable bias                & None \\
  Gradient checkpointing        & Enabled \\
  \bottomrule
\end{tblr}
\end{table}

\paragraph{Variant-specific settings.}
Table~\ref{tab:variant_hyper} lists the additional hyperparameters and
placement settings used by the LoRA variants and placement baselines. All
other settings follow Table~\ref{tab:default_hyper}.
For AdaLoRA, $r_{\mathrm{target}}$ denotes the target average rank, and
$t_{\mathrm{final}}$ denotes the number of final fixed-rank steps. Because
AdaLoRA uses a dynamic rank schedule, its realized rank and parameter count
may vary across runs.

\begin{table}[ht]
\centering
\caption{Variant-specific hyperparameters and placement settings.}
\label{tab:variant_hyper}
\begin{tblr}{
  width       = \linewidth,
  colspec     = {l X[l]},
  colsep      = 3pt,
  row{1}      = {font=\bfseries},
  hline{1,Z}  = {wd=0.08em},
  hline{2}    = {wd=0.05em},
}
  Variant & Additional settings \\
  LoRA+   & $\text{lr\_ratio} = 8.0$ (the learning rate for matrix $B$ is $8\times$ that for $A$) \\
  AdaLoRA & $r_{\mathrm{init}} = 64$, $r_{\mathrm{target}} = 8$, $T_{\mathrm{sched}} = 6{,}237$, $t_{\mathrm{init}} = 623$, $t_{\mathrm{final}} = 1{,}871$, $\Delta T = 10$, $\beta_1 = \beta_2 = 0.85$, $\lambda_{\mathrm{orth}} = 0.1$ \\
  GraLoRA & $k = 4$ \\
  PLoP & Selects $W_V$, $W_O$, and $W_{\mathrm{up}}$ via its neural-feature-norm (NFN) criterion \\
  IST & Dynamically selects $8$ of the $32$ layers to update at each step, using its default configuration \\
    
\end{tblr}
\end{table}

\paragraph{Evaluation prompts.}
We use the following task-specific prompts:
\begin{itemize}
  \item \textbf{MATH}~\citep{hendrycks2021measuring} and
  \textbf{GSM8K}~\citep{cobbe2021training}:
  ``Please solve this problem step by step. Put your final answer in
  \textbackslash boxed\{\}.''

  \item \textbf{CommonsenseQA}~\citep{talmor2019commonsenseqa}:
  ``Let's think step by step. Please format the final answer at the end
  of the response as: The answer is \{answer\}.''

  \item \textbf{MMLU}~\citep{hendrycks2021mmlu} and
  \textbf{LogiQA}~\citep{liu2020logiqa}:
  output only the option letter.

  \item \textbf{TruthfulQA}~\citep{lin2022truthfulqa}:
  use the question-and-options prompt in the evaluation data, ending
  with \texttt{Answer:}, without an additional task-specific instruction.

  \item \textbf{TyDiQA}~\citep{clark2020tydi}:
  extract the exact shortest answer span from the context without
  translation, formatted as \texttt{Answer: \{answer\}}.

  \item \textbf{HumanEval+}~\citep{liu2023is}:
  output only the Python function definition, without explanations,
  examples, markdown code blocks, or comments.
\end{itemize}

\section{Per-Seed Results on General Instruction Tuning}
\label{app:multiseed}

To make the averages reported in Table~\ref{tab:placement} traceable, we list the three per-seed results and their mean $\pm$ sample standard deviation below, using the same training settings as Appendix~\ref{app:exp_details}.

\begin{table*}[ht]
\centering
\caption{Per-seed general instruction tuning results. \ding{51} and \ding{55} denote DomLoRA and vanilla LoRA, respectively; shaded italic rows report the mean $\pm$ sample standard deviation over the three seeds.}
\label{tab:app_multiseed}

\footnotesize
\begin{tblr}{
  width       = \linewidth,
  colspec     = {
    Q[l,m,1.0]
    Q[c,m,0.6]
    Q[c,m,0.8]
    Q[c,m,1.1]
    Q[c,m,1.1]
    Q[c,m,1.1]
    Q[c,m,1.35]
    Q[c,m,1.1]
    Q[c,m,1.1]
    Q[c,m,1.1]
  },
  colsep      = 2.2pt,
  row{1}      = {font=\bfseries},
  row{5,13}   = {bg=gray!15,font=\itshape},
  row{9,17}   = {font=\itshape},
  cell{2}{1}  = {r=8}{},
  cell{10}{1} = {r=8}{},
  cell{2}{2}  = {r=4}{},
  cell{6}{2}  = {r=4}{},
  cell{10}{2} = {r=4}{},
  cell{14}{2} = {r=4}{},
}
  \toprule
  Model & Dom & Seed & MMLU & TyDiQA & CQA & TruthfulQA & GSM8K & LogiQA & Avg. \\
  \midrule
  Qwen  & \ding{51} & 42   & 71.7 & 49.7 & 81.7 & 57.5 & 92.4 & 57.5 & 68.4 \\
        &           & 777  & 72.1 & 51.1 & 82.4 & 57.0 & 92.6 & 57.1 & 68.7 \\
        &           & 979  & 71.9 & 50.8 & 81.0 & 57.6 & 92.9 & 57.1 & 68.6 \\
        &           & Mean & 71.9$\pm$0.2 & 50.5$\pm$0.7 & 81.7$\pm$0.7 & 57.4$\pm$0.3 & 92.6$\pm$0.3 & 57.2$\pm$0.2 & 68.6$\pm$0.2 \\
  \hline
        & \ding{55} & 42   & 71.7 & 47.1 & 80.8 & 48.8 & 92.0 & 55.9 & 66.1 \\
        &           & 777  & 70.4 & 53.5 & 80.4 & 55.1 & 92.4 & 55.8 & 67.9 \\
        &           & 979  & 71.2 & 54.1 & 79.7 & 55.8 & 91.7 & 57.1 & 68.3 \\
        &           & Mean & 71.1$\pm$0.7 & 51.6$\pm$3.9 & 80.3$\pm$0.6 & 53.2$\pm$3.9 & 92.0$\pm$0.4 & 56.3$\pm$0.7 & 67.4$\pm$1.2 \\
  \hline
  LLaMA & \ding{51} & 42   & 64.1 & 49.4 & 76.9 & 56.9 & 86.7 & 41.3 & 62.6 \\
        &           & 777  & 63.3 & 50.8 & 74.9 & 56.1 & 83.5 & 41.9 & 61.8 \\
        &           & 979  & 62.9 & 48.8 & 74.9 & 54.7 & 84.5 & 40.3 & 61.0 \\
        &           & Mean & 63.4$\pm$0.6 & 49.7$\pm$1.0 & 75.6$\pm$1.2 & 55.9$\pm$1.1 & 84.9$\pm$1.6 & 41.2$\pm$0.8 & 61.8$\pm$0.8 \\
  \hline
        & \ding{55} & 42   & 63.8 & 49.8 & 76.4 & 40.0 & 78.3 & 38.1 & 57.7 \\
        &           & 777  & 63.6 & 55.3 & 76.7 & 42.0 & 77.6 & 38.4 & 58.9 \\
        &           & 979  & 62.9 & 51.9 & 76.3 & 41.6 & 77.7 & 38.1 & 58.1 \\
        &           & Mean & 63.4$\pm$0.5 & 52.3$\pm$2.8 & 76.5$\pm$0.2 & 41.2$\pm$1.1 & 77.9$\pm$0.4 & 38.2$\pm$0.2 & 58.2$\pm$0.6 \\
  \bottomrule
\end{tblr}
\end{table*}

As shown in Table~\ref{tab:app_multiseed}, DomLoRA obtains a higher average than vanilla LoRA at all three seeds on both backbones, with mean averages of 68.6 versus 67.4 on Qwen3-8B and 61.8 versus 58.2 on LLaMA-3.1-8B-Instruct. The per-benchmark comparison is not uniform, however, because TyDiQA and CQA favor LoRA for some seeds, whereas TruthfulQA and GSM8K favor DomLoRA in every run. Thus, the three-seed results support a consistent advantage in the reported average under the shared training configuration rather than uniform improvements on every benchmark.

\section{General Instruction Tuning with LoRA Variants (RQ2)}
\label{sec:variant_results}

Table~\ref{tab:variants} reports results for LoRA variants on general instruction tuning across Qwen3-8B and LLaMA-3.1-8B-Instruct. Replacing broad placement with dominant-module placement reduces trainable parameters by two orders of magnitude and still improves the average score for most variants. The largest gain is obtained by DoRA on LLaMA-3.1-8B-Instruct, whose average score increases from 57.6 to 61.6; GraLoRA also improves by 3.9 points. On Qwen3-8B, DoRA and AdaLoRA improve from 66.4 to 68.5 and from 67.1 to 69.3, respectively, while NoRM is the only variant that decreases, by 0.2 points.
\begin{table*}[ht]
\centering
\caption{LoRA variant results on general instruction tuning across Qwen3-8B and LLaMA-3.1-8B-Instruct. \colorbox{gray!15}{Shaded rows} indicate dominant-module placement, and \textbf{bold} marks the best result within each method pair.}
\label{tab:variants}

\footnotesize
\begin{tblr}{
  width       = \linewidth,
  colspec     = {
    Q[l,m,1.0]
    Q[l,m,1.7]
    Q[c,m,0.6]
    Q[c,m,0.9]
    Q[c,m,0.9]
    Q[c,m,0.9]
    Q[c,m,0.6]
    Q[c,m,1.2]
    Q[c,m,1.2]
    Q[c,m,0.9]
    Q[c,m,0.6]
  },
  colsep      = 2.0pt,
  row{1}      = {font=\bfseries},
  row{3,5,7,9,11,13,15,17,19,21} = {bg=gray!15},
  cell{2}{1}  = {r=10}{font=\bfseries},
  cell{12}{1} = {r=10}{font=\bfseries},
  cell{2}{2}  = {r=2}{},
  cell{4}{2}  = {r=2}{},
  cell{6}{2}  = {r=2}{},
  cell{8}{2}  = {r=2}{},
  cell{10}{2} = {r=2}{},
  cell{12}{2} = {r=2}{},
  cell{14}{2} = {r=2}{},
  cell{16}{2} = {r=2}{},
  cell{18}{2} = {r=2}{},
  cell{20}{2} = {r=2}{},
}
  \toprule
  Model        & Method                            & Dom        & Params        & MMLU          & TyDiQA        & CQA           & TruthfulQA    & GSM8K         & LogiQA        & Avg.          \\
  \midrule
  Qwen3        & NoRM & \ding{55}  & 156.7M        & 71.3          & \textbf{51.7} & \textbf{81.4} & \textbf{62.7} & 93.3          & 56.8          & \textbf{69.5} \\
               &                                   & \ding{51}  & \textbf{0.9M} & \textbf{71.5} & 48.0          & 81.2          & 62.4          & \textbf{94.0} & \textbf{58.4} & 69.3          \\
  \cline{2-11}
               & DoRA          & \ding{55}  & 176.0M        & 71.4          & 46.3          & 78.8          & 53.9          & 91.8          & 56.4          & 66.4          \\
               &                                   & \ding{51}  & \textbf{1.1M} & \textbf{71.6} & \textbf{50.0} & \textbf{82.0} & \textbf{57.0} & \textbf{92.6} & \textbf{57.5} & \textbf{68.5} \\
  \cline{2-11}
               & LoRA+  & \ding{55}  & 174.6M        & 70.6          & 48.7          & \textbf{80.3} & 55.7          & \textbf{92.9} & \textbf{57.8} & 67.7          \\
               &                                   & \ding{51}  & \textbf{1.0M} & \textbf{70.8} & \textbf{48.9} & 80.2          & \textbf{60.8} & \textbf{92.9} & 55.9          & \textbf{68.3} \\
  \cline{2-11}
               & AdaLoRA  & \ding{55}  & 22.3M         & \textbf{71.1} & \textbf{50.5} & 80.2          & 52.3          & 91.9          & 56.5          & 67.1          \\
               &                                   & \ding{51}  & \textbf{0.1M} & 70.4          & 48.6          & \textbf{81.6} & \textbf{63.2} & \textbf{93.4} & \textbf{58.4} & \textbf{69.3} \\
  \cline{2-11}
               & GraLoRA  & \ding{55}  & 174.6M        & 71.1          & 50.0          & 79.9          & \textbf{54.8} & \textbf{92.0} & 55.3          & 67.2          \\
               &                                   & \ding{51}  & \textbf{1.0M} & \textbf{71.8} & \textbf{50.3} & \textbf{81.6} & 53.0          & 91.9          & \textbf{57.8} & \textbf{67.7} \\
  \hline
  LLaMA        & NoRM  & \ding{55}  & 141.9M        & \textbf{65.7} & \textbf{50.7} & \textbf{78.5} & 52.0          & 86.4          & 41.0          & \textbf{62.4} \\
               &                                   & \ding{51}  & \textbf{1.1M} & 64.9          & 49.1          & 75.9          & \textbf{54.7} & \textbf{86.8} & \textbf{43.2} & \textbf{62.4} \\
  \cline{2-11}
               & DoRA          & \ding{55}  & 169.9M        & \textbf{63.6} & \textbf{49.7} & \textbf{75.8} & 40.6          & 78.2          & 37.8          & 57.6          \\
               &                                   & \ding{51}  & \textbf{1.2M} & 63.4          & 49.2          & 74.3          & \textbf{55.8} & \textbf{86.1} & \textbf{40.6} & \textbf{61.6} \\
  \cline{2-11}
               & LoRA+   & \ding{55}  & 167.8M        & 63.0          & 48.5          & 73.5          & \textbf{48.3} & 81.7          & 38.4          & 58.9          \\
               &                                   & \ding{51}  & \textbf{1.2M} & \textbf{65.0} & \textbf{49.7} & \textbf{76.0} & 44.8          & \textbf{86.7} & \textbf{41.8} & \textbf{60.7} \\
  \cline{2-11}
               & AdaLoRA  & \ding{55}  & 21.1M         & 64.1          & 47.8          & \textbf{76.1} & \textbf{51.3} & 84.1          & 39.2          & 60.4          \\
               &                                   & \ding{51}  & \textbf{0.1M} & \textbf{65.1} & \textbf{49.4} & 75.2          & 50.7          & \textbf{84.4} & \textbf{42.2} & \textbf{61.2} \\
  \cline{2-11}
               & GraLoRA   & \ding{55}  & 167.8M        & 63.7          & 48.5          & 74.1          & 45.2          & 80.5          & 37.9          & 58.3          \\
               &                                   & \ding{51}  & \textbf{1.2M} & \textbf{64.3} & \textbf{49.2} & \textbf{75.8} & \textbf{55.4} & \textbf{85.1} & \textbf{43.6} & \textbf{62.2} \\
  \bottomrule
\end{tblr}
\end{table*}

\section{Sensitivity to the LoRA Rank and Learning Rate (A1)}
\label{app:rank_ablation}

We examine the sensitivity of DomLoRA to the learning rate and the LoRA rank--scaling factor pair on LLaMA-3.1-8B. Table~\ref{tab:lr_ablation} sweeps the learning rate with ($r$,\,$\alpha$) fixed to (64,\,128). The average score peaks at $2e-5$; a tenfold larger rate drops the average to 56.4, with the largest drop on CommonsenseQA. Table~\ref{tab:rank_ablation} sweeps ($r$,\,$\alpha$) with $\alpha=2r$ at learning rate $2e-5$. The average score varies from 61.7 to 62.8 across the three settings. These results indicate that the performance gain of dominant-module placement does not rely on a specific LoRA rank, while the learning rate should be kept at the standard value.

\begin{table*}[h]
\centering
\caption{Sensitivity of DomLoRA to the learning rate on LLaMA-3.1-8B, with ($r$,\,$\alpha$) fixed to (64,\,128).}
\label{tab:lr_ablation}

\footnotesize
\begin{tblr}{
  width       = \linewidth,
  colspec     = {
    Q[c,m,0.9]
    Q[c,m,0.9]
    Q[c,m,0.8]
    Q[c,m,0.9]
    Q[c,m,1.1]
    Q[c,m,1.1]
    Q[c,m,1.2]
    Q[c,m,1.3]
    Q[c,m,1.1]
    Q[c,m,0.7]
  },
  colsep      = 2.5pt,
  rowsep      = 1pt,
  row{1}      = {font=\bfseries},
}
  \toprule
  LR & MMLU & TyDiQA & CQA & TruthfulQA & GSM8K & MATH500 & HumanEval+ & MT-Bench & Avg. \\
  \midrule
  $2e-4$ & 63.1 & 46.9 & 43.7 & 57.6 & 82.9 & 44.6 & 54.3 & 58.0 & 56.4 \\
  $2e-5$ & 64.1 & 49.4 & 76.9 & 56.9 & \textbf{86.4} & \textbf{48.2} & \textbf{59.1} & 61.3 & \textbf{62.8} \\
  $5e-6$ & \textbf{64.6} & \textbf{50.2} & 76.4 & 48.8 & 86.9 & 48.0 & 54.9 & \textbf{62.6} & 61.6 \\
  \bottomrule
\end{tblr}
\end{table*}

\begin{table*}[h]
\centering
\caption{Sensitivity of DomLoRA to the LoRA rank and scaling factor on LLaMA-3.1-8B, with $\alpha = 2r$ and learning rate $2e-5$.}
\label{tab:rank_ablation}

\footnotesize
\begin{tblr}{
  width       = \linewidth,
  colspec     = {
    Q[c,m,0.7]
    Q[c,m,0.9]
    Q[c,m,0.8]
    Q[c,m,0.9]
    Q[c,m,1.1]
    Q[c,m,1.1]
    Q[c,m,1.2]
    Q[c,m,1.3]
    Q[c,m,1.1]
    Q[c,m,0.7]
  },
  colsep      = 2.5pt,
  rowsep      = 1pt,
  row{1}      = {font=\bfseries},
}
  \toprule
  ($r$,\,$\alpha$) & MMLU & TyDiQA & CQA & TruthfulQA & GSM8K & MATH500 & HumanEval+ & MT-Bench & Avg. \\
  \midrule
  (16,\,32) & 64.8 & 49.8 & \textbf{79.2} & 55.2 & 85.9 & 47.2 & 54.9 & \textbf{61.6} & 62.3 \\
  (32,\,64) & \textbf{64.8} & \textbf{51.2} & 77.1 & 53.9 & 85.1 & 45.4 & 54.9 & 60.9 & 61.7 \\
  (64,\,128) & 64.1 & 49.4 & 76.9 & \textbf{56.9} & \textbf{86.4} & \textbf{48.2} & \textbf{59.1} & 61.3 & \textbf{62.8} \\
  \bottomrule
\end{tblr}
\end{table*}

\section{PAGE Stability across Seeds and Sample Counts (A2)}
\label{app:page_stability}

Section~\ref{sec:observation} shows that the dominant adaptation module already appears at Step~0 and persists throughout fine-tuning. This section examines whether the same module is identified when the data subset used for PAGE computation varies.

\paragraph{Data-order seeds.} We repeat the Step~0 PAGE evaluation on both backbones with five data-order seeds (42, 777, 979, 1234, and 31337). Each seed shuffles the T\"{u}lu V3 training set and draws a fresh subset of 32 samples; all other settings follow Section~\ref{sec:observation}. Since the PAGE score scale varies with the sample subset, we compare module rankings rather than raw scores. Table~\ref{tab:page_multiseed} reports the top-two modules for each seed. The Layer~6 down-projection of Qwen3-8B ranks first in all five seeds, with a margin of roughly $10\times$ over the runner-up, and the Layer~1 down-projection of LLaMA-3.1-8B ranks first in all five seeds.

\begin{table}[h]
\caption{Top-ranked modules under five data-order seeds at Step~0.}
\label{tab:page_multiseed}
\begin{center}
\small
\begin{tabular}{lccc}
\toprule
\textbf{Backbone} & \textbf{Seed} & \textbf{Top-1} & \textbf{Top-2} \\
\midrule
Qwen3-8B & 42 & $\bm{W}_{\mathrm{down}}$@6 & $\bm{W}_{\mathrm{v}}$@22 \\
 & 777 & $\bm{W}_{\mathrm{down}}$@6 & $\bm{W}_{\mathrm{v}}$@22 \\
 & 979 & $\bm{W}_{\mathrm{down}}$@6 & $\bm{W}_{\mathrm{v}}$@22 \\
 & 1234 & $\bm{W}_{\mathrm{down}}$@6 & $\bm{W}_{\mathrm{v}}$@22 \\
 & 31337 & $\bm{W}_{\mathrm{down}}$@6 & $\bm{W}_{\mathrm{v}}$@22 \\
\midrule
LLaMA-3.1-8B-Instruct & 42 & $\bm{W}_{\mathrm{down}}$@1 & $\bm{W}_{\mathrm{v}}$@0 \\
 & 777 & $\bm{W}_{\mathrm{down}}$@1 & $\bm{W}_{\mathrm{v}}$@1 \\
 & 979 & $\bm{W}_{\mathrm{down}}$@1 & $\bm{W}_{\mathrm{v}}$@0 \\
 & 1234 & $\bm{W}_{\mathrm{down}}$@1 & $\bm{W}_{\mathrm{v}}$@0 \\
 & 31337 & $\bm{W}_{\mathrm{down}}$@1 & $\bm{W}_{\mathrm{v}}$@0 \\
\bottomrule
\end{tabular}
\end{center}
\end{table}

\paragraph{Number of samples.} Figure~\ref{fig:page_sample_sweep} varies the number of PAGE samples from 1 to 32 across the four fine-tuning datasets. The Layer~6 down-projection of Qwen3-8B remains the top-1 candidate in all 128 dataset--sample combinations, and the Layer~1 down-projection of LLaMA-3.1-8B-Instruct in 127 of 128; the single exception is KodCode-V1-SFT-4o with one sample, where $\bm{W}_{\mathrm{v}}$ at Layer~0 is marginally higher. PAGE-based selection is therefore robust to the number of samples, and the probe itself is cheap: 32 samples take 80--120 seconds on a single H100 GPU, a negligible overhead relative to fine-tuning.

\begin{figure}[ht]
    \centering
    \includegraphics[width=\linewidth]{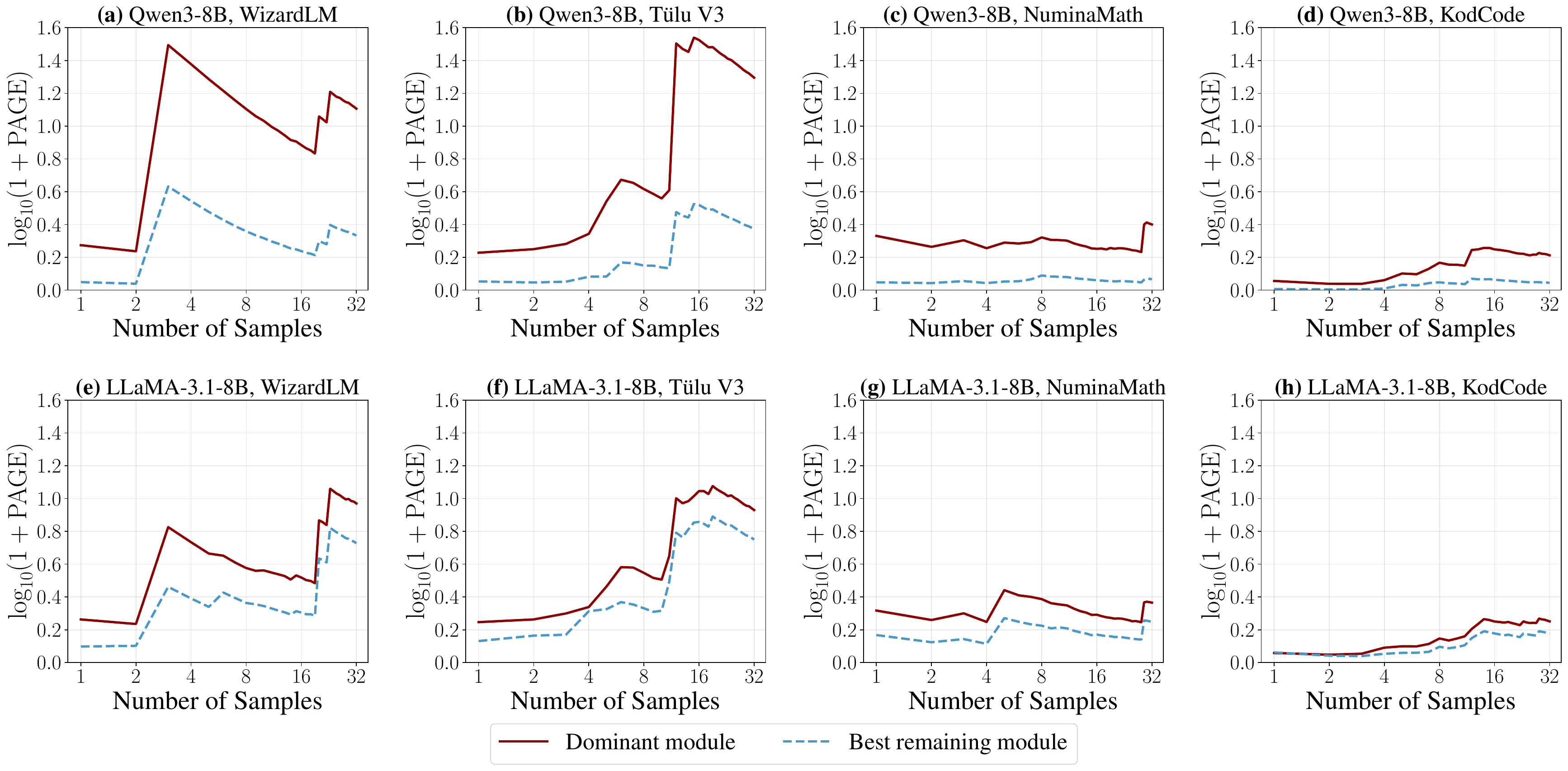}
    \caption{\textbf{PAGE under varying sample sizes.} Each panel compares the dominant module with the best remaining module as the number of samples varies from 1 to 32. The top row shows Qwen3-8B, and the bottom row shows LLaMA-3.1-8B-Instruct.}
    \label{fig:page_sample_sweep}
\end{figure}

\end{document}